\documentclass[letterpaper]{article} 
\usepackage{aaai23}  
\usepackage{times}  
\usepackage{helvet}  
\usepackage{courier}  
\usepackage[hyphens]{url}  
\usepackage{graphicx} 
\usepackage{natbib}  
\usepackage{caption} 
\usepackage{algorithm}

\usepackage{newfloat}
\usepackage{listings}
\DeclareCaptionStyle{ruled}{labelfont=normalfont,labelsep=colon,strut=off} 
\floatstyle{ruled}
\newfloat{listing}{tb}{lst}{}
\floatname{listing}{Listing}
\usepackage{lipsum}
\usepackage{xspace}
\usepackage{amsmath}
\usepackage[capitalise,noabbrev,nameinlink]{cleveref}
\creflabelformat{equation}{#2\textup{#1}#3}
\usepackage[inline]{enumitem}
\usepackage{algpseudocode}

\usepackage{subfig}
\usepackage{booktabs}
\usepackage{multicol} 

\usepackage[most]{tcolorbox}
\newtcolorbox{leftvrule}[1][]{colback=white,
boxrule=0pt, boxsep=0pt, breakable, enhanced jigsaw, borderline west={1.5pt}{0pt}{black},
before skip=10pt,after skip=10pt,
#1}

\usepackage{amsfonts}
\DeclareMathOperator{\E}{\mathbb{E}}
\DeclareMathOperator*{\bmax}{\overset \bullet {max}}
\DeclareMathOperator{\mv}{MV}
\DeclareMathOperator{\rb}{\delta}
\DeclareMathOperator*{\argmax}{arg~max}

\newcommand{\algName}{Dyna-ATMQ\xspace}
\newcommand{\algNameShort}{ATMQ\xspace}
\newcommand{\gitLink}{\url{https://github.com/LAVA-LAB/ATM}}

\usepackage{amsthm}
\newtheorem{theorem}{Theorem}
\newtheorem{lemma}{Lemma}
\newtheorem*{lemma*}{Lemma}

\newcommand{\resizetable}[1]{\scalebox{0.8}{#1}}
\setlist[enumerate]{label={\arabic*)}}

\title{Act-Then-Measure: Reinforcement Learning for Partially Observable Environments with Active Measuring}

\author {
    Merlijn Krale,
    Thiago D. Simão,
    Nils Jansen
}
\affiliations {
    Radboud University, Nijmegen\\
    Institute for Computing and Information Sciences\\
    \{merlijn.krale, thiago.simao, nils.jansen\}@ru.nl
}

\begin{document}

\maketitle

\begin{abstract}

We study Markov decision processes (MDPs), where agents have direct control over when and how they gather information, as formalized by action-contingent noiselessly observable MDPs (ACNO-MPDs).
In these models, actions consist of two components: a \emph{control action} that affects the environment, and a \emph{measurement action} that affects what the agent can observe.
To solve ACNO-MDPs, we introduce the \emph{act-then-measure} (ATM) heuristic, which assumes that we can ignore future state uncertainty when choosing control actions.
We show how following this heuristic may lead to shorter policy computation times and prove a bound on the performance loss incurred by the heuristic.
To decide whether or not to take a measurement action, we introduce the concept of \emph{measuring value}.
We develop a reinforcement learning algorithm based on the ATM heuristic, using a Dyna-Q variant adapted for partially observable domains, and showcase its superior performance compared to prior methods on a number of partially-observable environments.

\end{abstract}

\section{Introduction}
In recent years, partially observable Markov decision processes (POMDPs) have become more and more widespread to model real-life situations involving uncertainty \citep{DBLP:journals/robotics/KormushevCC13,DBLP:journals/comsur/LeiTZLZS20,DBLP:journals/tits/SunbergK22}.
\emph{Active measure} POMDPs are an interesting subset of these environments, in which agents have direct control over when and how they gather information, but gathering information has an associated cost~\citep{DBLP:conf/ai/BellingerC0T21}.
For example, maintenance of a sewer system might require regular inspections \citep{jimenez2022deterioration}, or appropriate healthcare might require costly or invasive tests to be run \citep{DBLP:journals/csur/YuLNY23}.
In both cases, the risk or cost of gaining information needs to be weighted against the value of such information.

Reinforcement learning (RL) is a promising approach to handling problems where we must actively gather information.
However, due to the complexity of POMDPs, successes with RL methods in partially observable settings are still limited \citep{DBLP:journals/ml/Dulac-ArnoldLML21}.
One may circumvent this by focusing on  subsets of POMDPs which have certain exploitable properties.
For example, \citet{DBLP:conf/aistats/GuoDB16} proposed an efficient RL algorithm for small-horizon POMDPs, \citet{Simao2023spipomdp} investigates offline RL where finite histories provide sufficient statistics. 
Similarly, this paper focuses on a subset of active measure POMDPs with complete and noiseless observations, called  \emph{action contingent noiselessly observable} MDPs~\citep[ACNO-MDPs;][]{DBLP:conf/nips/NamFB21}.

For ACNO-MDPs, two RL algorithms already exist.
The first, AMRL-Q \citep{DBLP:conf/ai/BellingerC0T21}, is computationally inexpensive but uses a most-likely state approximation and always converges to non-measuring policies, causing poor performance in stochastic environments.
In contrast, the \emph{observe-then-plan} framework proposed by \citet{DBLP:conf/nips/NamFB21} performs well in smaller stochastic environments, but its reliance on general POMDP planners for policy optimization makes it computationally expensive.
Therefore, \emph{we investigate lightweight and high-performing reinforcement learning methods in stochastic ACNO-MDPs}.

In this paper, we propose a method for stochastic ACNO-MDPs\footnote{Stochastic MDPs are the opposite of deterministic MDPs where all probability distributions are Dirac.} in which we explicitly use knowledge of the setting for both learning and exploitation.
To this end, we propose the \textit{act-then-measure} heuristic, inspired by the $Q_{\text{MDP}}$ approach \cite{DBLP:conf/icml/LittmanCK95}, which drastically decreases policy computation times.
Since our method relies on a heuristic to compute the policy, we also investigate how much performance we can lose compared to the optimal policy, for which we prove an upper bound.
We then describe an algorithm based on Dyna-Q which uses this heuristic for RL in ACNO-MDPs.
We compare it empirically to previous methods, both in an environment designed to test whether algorithms can accurately determine the value of measuring, and in a standard RL environment.
In both, we find our algorithm outperforms AMRL-Q and \emph{observe-then-plan}, while staying computationally tractable for much bigger environments than the latter.

\paragraph{Contributions.}
The main contributions of this work are:
\begin{enumerate*}
    \item identifying limitations of previous RL approaches for ACNO-MPDs,
    \item introducing the \emph{act-then-measure} (ATM) heuristic,
    \item introducing the concept of \emph{measuring value}, and
    \item implementing \algName, an RL algorithm for ACNO-MDPs following the ATM heuristic.
\end{enumerate*}

\section{Background}
\label{sec:setting}
This section gives a formal description of ACNO-MDPs, then describes and analyzes RL methods for the setting.

\subsection{ACNO-MDPs}

\begin{figure}[tbp]
    \centering
    \includegraphics[width=0.7\columnwidth]{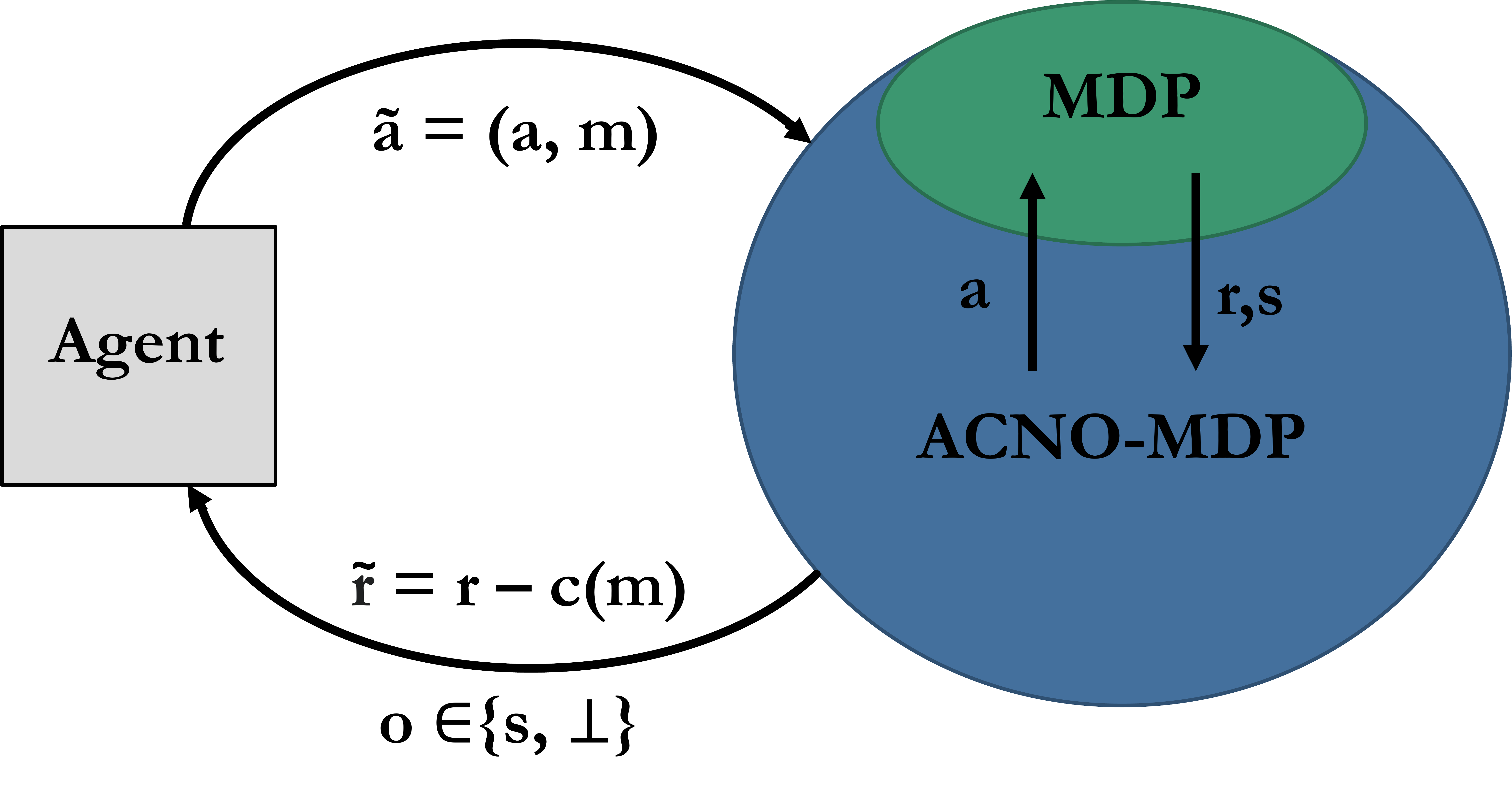}
    \caption{Agent-environment interaction in an ACNO-MDP. 
    The agent performs a control action $a$ and measurement $m$ at each time step $t$.
    The internal environment state is defined by an MDP and affected only by control actions.
    After each step, the agent receives a scalarized reward $\tilde{r} = r-C(m)$ and observation $o \in \{s, \bot\}$ (with $o=s \iff m=1$).}
    \label{fig:teaser}
\end{figure}

We define our problem as an \textit{action-contingent noiselessly observable MDP} \citep[ACNO-MDP;][]{DBLP:conf/nips/NamFB21}.
An ACNO-MDP is defined by a tuple $ \mathcal{M} = (S,\tilde{A}{=}A {\times} M,P,R,C,\Omega, O, \gamma)$, where $(S,A,P,R,\gamma)$ are the components of a standard MDP: $S$ is the state space, $A$ is the action space, $P(s'\mid s,a)$ is the transition function, $R(s,a)$ is the reward function and $\gamma \in [0,1]$ is the discount factor.
However, in the ACNO-MDP framework $\tilde{A}$ consists of pairs of \emph{control actions} and \emph{measurements}, taking the form $\tilde{a}=\langle a, m \rangle \in A \times M$, where $M = \{\text{not observe, observe}\}=\{0,1\}$.
A control action $a \in A$ affects the environment, while the measurement choice $m \in M$ only affects what the agent observes.
Following the typical notation from POMDPs, $\Omega$ is the observation space and $O$ the observation function, so $O(o\mid s',\langle a,m\rangle )$ is the probability of receiving observation $o \in \Omega$ when taking measurement $m$ and action $a$, after transitioning to the state $s'$.
In ACNO-MDPs all measurements are complete and noiseless, so we can define $\Omega = {S} \cup \{ \bot \}$, where $\bot$ indicates an empty observation.
Then, the observation function is defined as
$O(o\mid s',\langle a,1 \rangle ) = 1 \iff o=s'$, and $0$ otherwise.
Similarly, $O(o\mid s',\langle a,0 \rangle) = 1 \iff o=\bot$, and $0$ otherwise.
Measuring has an associated cost $C(0) = 0$ and $C(1) = c$ (with $c \geq 0$), which gets subtracted from our reward, giving us a \textit{scalarized-reward} $\tilde{r}_t = R(s_t,a_t) - C(m_t)$.

Agent-environment interactions for ACNO-MDPs are visualized in \cref{fig:teaser}.
Starting in some initial state $s_0$, for each time-step $t$ the agent executes an action-pair $\langle a_t, m_t \rangle$ according to a policy $\pi$.
In general, these policies are defined for a belief state $b_t$, a distribution over the states representing the probability of being in each state of the environment, summarising all past interactions.
After executing $\langle a_t, m_t \rangle$ in $s_t$, the environment transitions to a new state $s_{t+1} \sim P(\cdot \mid s_t, a_t)$, and returns to the agent a reward $r_t = R(s_t,a_t)$, a cost $c_t = C(m_t)$ and observation $o_{t+1} \sim O(\cdot \mid s_{t+1},\langle  a_t,m_t\rangle )$.
The goal of the agent is to compute a policy $\pi$ with the highest expected total discounted scalarized-reward
$
  V(\pi, \mathcal{M}) = \E_{\pi,\mathcal{M}} \left[ \sum_t \gamma^t \tilde{r}_t \right].
$

In this paper, we will mainly focus on reinforcement learning in ACNO-MDPs.
We assume the agent only has access to the total number of states and the signals returned by the environment in each interaction, but otherwise has no prior information about the dynamics of the environment.

\subsection{Q-learning for ACNO-MDPs}

\citet{DBLP:conf/ai/BellingerC0T21} propose to solve the ACNO-MDP problem using an adaptation of Q-learning \citep{DBLP:journals/ml/WatkinsD92}.
To choose the best action pair, the agent estimates both the transition probability function and value functions with tables $\hat{P}$ and $Q$ of sizes $|S \times A \times S|$ and $|S \times \tilde{A}|$, respectively.
Both are initialized uniformly, except that all actions with $m=1$ are given an initial bias in $Q$ to promote measuring in early episodes.

Beginning at the initial state, for every state $s_t$ the agent executes an $\epsilon$-greedy action-pair $\langle a_t,m_t \rangle$ according to $Q$.
When $m_t=1$, the successor state $s'=s_{t+1}$ is observed so the algorithm updates the transition probability $\hat{P}( \cdot \mid s_t, a_t)$.
When $m_t=0$, AMRL-Q does not update $\hat{P}$ and assumes the successor state is the \textit{most likely next state} according to $\hat{P}$:
\[
    s' = \argmax_{s\in S} \hat{P}( s \mid s_t, a_t).
\]
Using the reward $r_t$ and the (potentially estimated) successor state $s'$, AMRL-Q updates both $Q(s_t,\langle a_t,0 \rangle)$ and $Q(s_t,\langle a_t,1 \rangle)$, as follows:
\begin{equation}
\begin{split}
    Q(s_t,& \langle a_t, m\rangle) \leftarrow  (1-\alpha)  Q(s_t, \langle a_t, m\rangle) \\  &+ \alpha \left[r_t - C(m) + \gamma \max_{a',m'}Q(s', \langle a', m' \rangle)\right].
\end{split}
\end{equation}

Although AMRL-Q is conceptually interesting and has very low computation times, in practice the algorithm has some considerable shortcomings:

\begin{figure}[tb]
\centering
\includegraphics[width=0.6\columnwidth]{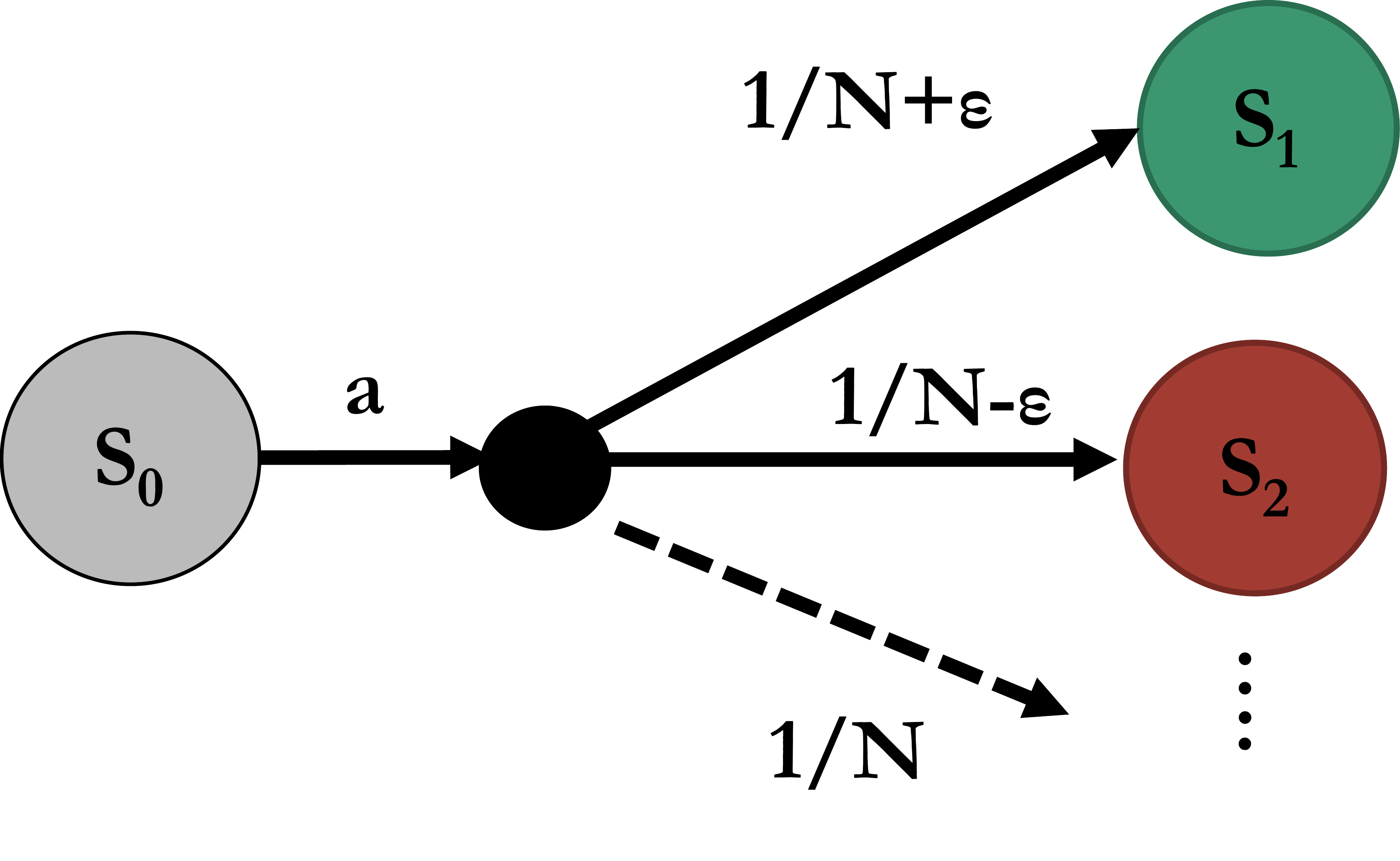} 
\caption{An ACNO-MDP where the value $Q(s_0, \langle a, 0\rangle )$ based on the most-likely successor state can be made arbitrarily inaccurate. In this example, using a most likely state means considering only $s_1$, even though the probability of reaching this state is only $1/N + \epsilon$, with $N$ the number of successor states (which is only bounded by $|S|$), neglecting the probability of reaching the remaining successor states.}
\label{fig:MLFailExample}
\end{figure}

\paragraph{AMRL-Q does not measure after convergence.}
Apart from its $\epsilon$-greediness,  for any state $s$ AMRL-Q takes a measuring actions $\langle a,1 \rangle$ if this action-measurement pair has the highest Q-value.
In particular, this means that $Q(s,\langle a,1 \rangle) > Q(s,\langle a,0 \rangle)$ must hold.
However, since these Q-values get updated simultaneously and with the same $r_t$ and $s'$, but $(r_t - C(m))$ is always lower for $m=1$, $Q(s,\langle a,1 \rangle)$ converges to a value lower than $Q(s,\langle a,0 \rangle)$.
This means AMRL-Q only converges to non-measuring policies, which is suboptimal for those stochastic environments where the optimal policy requires taking measurements.

\paragraph{AMRL-Q ignores the state uncertainty.}
As visualized in \cref{fig:MLFailExample}, the most-likely successor state used in AMRL-Q can give arbitrarily inaccurate approximations of the current state.
Apart from sub-optimal action selection, this may also cause inaccuracies in the model in later steps, since AMRL-Q makes no distinction between measured and non-measured states for model updates.

\subsection{Solving ACNO-MDP via POMDPs}
\citet{DBLP:conf/nips/NamFB21} introduce two frameworks for solving tabular ACNO-MDPs.
The first, named \emph{observe-before-planning}, has an initial exploration phase in which the agent always measures to learn an approximated model.
After this phase, a generic POMDP-solver computes a policy based on the approximated model.
The second framework, named \emph{observe-while-planning}, starts by using a POMDP-solver on some initial model from the start, and updates the model on-the-fly based on the measurements made.
For both frameworks a specific implementation is tested, using \emph{episodic upper lower exploration in reinforcement learning} \citep[EULER;][]{DBLP:conf/icml/ZanetteB19} for the exploration phase and \emph{partially observable Monte-Carlo planning} \citep[POMCP;][]{DBLP:conf/nips/SilverV10} as a generic POMDP-solver.
Both algorithms outperform the tested generic POMDP RL-method, with \emph{observe-before-planning} performing slightly better overall.
We therefor focus on this framework in this paper.
Apart from some more specific disadvantages of using POMCP for ANCO-MDPs (which we describe more fully in \cref{sec:ACNOChanges}), we note one general shortcoming of this framework.

\paragraph{\emph{Observe-before-planning} only optimises information gathering.}
While \emph{observe-before-planning} makes explicit use of the ACNO-MDP structure in its exploration phase, for exploitation it relies only on a generic POMDP-solver.
These solvers generally have high computational complexity, which limits what environment they can be employed in.
In contrast, a method that uses the ACNO-MDP structure (where only control actions affect the underlying state) could in principle solve larger and more complex problems.

\begin{figure}[t]
\centering
\includegraphics[width=0.8\columnwidth]{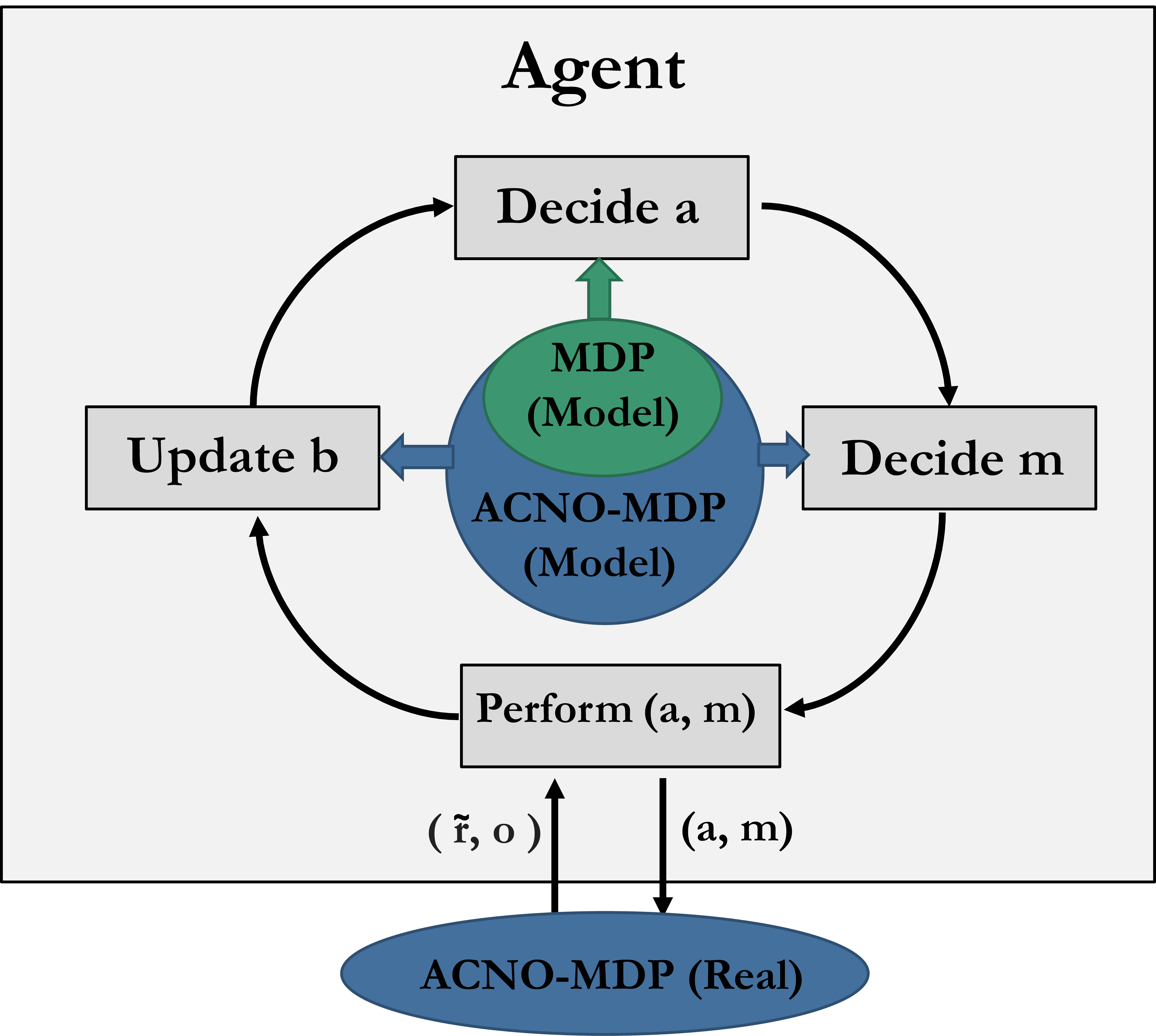} 
\caption{The control-loop for solving ACNO-MDPs using the \emph{act-then-measure} heuristic. At each timestep, a control actions $a$ is chosen according to the current belief state $b$, as though it is a belief over MDP-states. Then, a measurement $m$ is picked without ignoring future state uncertainty, $(a,m)$ is executed and the belief state $b$ is updated accordingly.}
\label{fig:ATMLoop}
\end{figure}

\section{The Act-Then-Measure Heuristic}
\label{sec:globalView}

In this section, we propose the \emph{act-then-measure} (ATM) heuristic for approximating optimal policies in ACNO-MDPs.
Intuitively, this heuristic is based on the observation that control actions and measurements have very different effects, which implies it might be desirable to choose them using separate processes.
Therfore, inspired by the $Q_{\text{MDP}}$ heuristic \citep{DBLP:conf/icml/LittmanCK95}, our heuristic \emph{chooses a control action, assuming all (state) uncertainty will be resolved in the next state(s).}

Following this heuristic, we do not need to consider measurements while deciding control actions, since measuring only affects state uncertainty.
This means we can use a basic control loop (\cref{fig:ATMLoop}), in which we choose control actions before measurements.
Moreover, for control actions computing future returns can be done using MDP approaches, which lets us use the following approximation:
\begin{equation}
\label{eq:QApproxATM}
    Q(b, a)  \approx \sum_{s \in S} b(s) Q_{\text{MDP}}(s,a),
\end{equation}
where $Q_{\text{MDP}}(s, a)$ is the value of taking action $a$ in state $s$ and following the optimal policy of the underlying MDP afterward, and $b$ denotes the current belief, so $b(s)$ is the probability that the current state is $s$.
Since, in general, MDPs are more tractable than POMDPs, this approximation allows for a more efficient policy computation than POMDP-based methods like \emph{observe-then-plan}.
At the same time, in contrast to AMRL-Q, belief states are not approximated, which means state uncertainty for the current belief state is fully considered.
Furthermore, measurements can be made after convergence, and future state uncertainty is considered when deciding whether to measure.

\subsection{Evaluating Measurements}
\label{sec:MeasurmentRegret}

To use the ATM heuristic, we need a principled way to determine whether to take a measurement.
Therefore, we require the ability to \textit{estimate the value of a measurement}.
For this, we start by defining the value function $Q_{\text{ATM}}(b,\tilde{a})$ as the value for executing $\tilde{a}$ in belief state $b$, assuming we follow the ATM-heuristic, i.e. that we choose control actions according to \cref{eq:QApproxATM}.
We will define $Q_{\text{ATM}}(b,\tilde{a})$ using Bellman equations.
For readability, we first introduce the following notations:
\begin{equation*}
\label{eq:b_next}
    b'(s' | b, a) = \sum_{s \in S} b(s) P(s'|s',a), \text{ and }
    \bmax_{\tilde{a} \in \tilde{A}}  = \max_{m \in M} \max_{a \in A},
\end{equation*}
where $b'(s' | b, a)$ represents the probability of transitioning to state $s'$ when taking action $a$ in the current belief state~$b$, and $\bmax$ describes the optimal action pair if the control action is decided before the measurement.

We note that the form of the Bellman equations for $Q_{\text{ATM}}(b, \tilde{a})$ depends on the current measuring action.
If measuring, we can use the information we gain to choose the optimal action to take, giving us the following:
\begin{equation}
\label{eq:bellmanMeasuring}
    Q_{\text{ATM}}(b, \langle a, 1 \rangle) = \hat{r} -c + \gamma \sum_{s' \in S} b'(s' | b, a) \bmax_{\tilde{a} \in \tilde{A}}  Q_{\text{ATM}}(s', \tilde{a}),
\end{equation}
with $\hat{r}$ the expected reward of taking action $a$ in belief state $b$ and $Q_{\text{ATM}}(s,\tilde{a})$ the Q-value of a belief state with $b(s) = 1$.
If not measuring, we can only base our next action on the expected next belief.
We may then define the \emph{belief-optimal action} $\tilde{a}_b$ as follows:
\begin{equation}
\label{eq:a_b}
\begin{split}
    \tilde{a}_b
        & = \arg\bmax_{\tilde{a} \in \tilde{A}} Q_{\text{ATM}}(b_\text{next}(b,a),\tilde{a} ) \\
        & = \arg\bmax_{\tilde{a} \in \tilde{A}} \sum_{s' \in S} b'(s' | b, a) Q_{\text{ATM}}(s', \tilde{a}),
\end{split}
\end{equation}
where the second equality follows from the fact that control actions are chosen in accordance to \cref{eq:QApproxATM}, and is proven in \cref{sec:proofs}.
Using this, we find the following Bellman equation for $m=0$:
\begin{equation}
\label{eq:bellmanNotMeasuring}
    Q_{\text{ATM}}(b, \langle a, 0 \rangle) {=} \hat{r} + \gamma \sum_{s' \in S} b'(s' | b, a) Q_{\text{ATM}}(s',\tilde{a}_b).
\end{equation}
Based on \cref{eq:bellmanMeasuring,eq:bellmanNotMeasuring}, we define the \emph{measuring value}~$\mv(b)$ as the difference between these two Q-values:
\begin{equation}
\label{eq:measureRegret}
\begin{split}
     &\mv(b,a) = Q_{\text{ATM}}(b,\langle a,1\rangle) - Q_{\text{ATM}}(b, \langle a, 0 \rangle) \\
     &{=} {-}c {+} \gamma \sum_{s \in S} b'(s | b, a)\!\! \left[  \bmax_{\tilde{a} \in \tilde{A}}\!Q_{\text{ATM}}(s,\tilde{a}) {-} Q_{\text{ATM}}(s,\tilde{a}_{b}) \right]
\end{split}
\end{equation}
To illustrate, suppose we predict a next belief state $b'$ as given in \cref{fig:MeasurementRegretExample}, and for simplicity assume $\gamma = 1$.
If we choose not to measure, the belief optimal action for $b'$ is $a_0$, yielding a reward of $0.8$ on average.
If instead, we do take a measurement, we can decide to take action $a_0$ if we reach state $s_0$ and action $a_1$ if we reach state $s_1$, yielding a return of $1 - c$.
Following \cref{eq:measureRegret}, the measuring value is thus $1-c-0.8=0.2-c$, meaning it is worth taking a measurement if $c \leq 0.2$.
Generalising this example, we find the following condition for taking measurements:
\begin{equation}
\label{eq:MeasureValueCondition}
m_{\mv}(b,a) = \begin{cases}
        1 & \text{if } \mv(b, a) \geq 0; \\
        0 & \text{otherwise},
    \end{cases}    
\end{equation}
and can define a policy following the ATM heuristic as:
\begin{equation}
\label{eq:piATM}
    \pi_{\text{ATM}}(b) = \langle \max_{a \in A} Q(b,a) . m_{\text{MV}}(b, \max_{a \in A} Q(b,a)) \rangle,
    \end{equation}
with $Q(b,a)$ as defined in \cref{eq:QApproxATM}.

In practice, calculating $Q_{\text{ATM}}(s,\tilde{a})$ in \cref{eq:bellmanMeasuring,eq:bellmanNotMeasuring} for all possible next belief states can be computationally intractable.
An intuitive (over-)approximation to use is $Q_{\text{ATM}}(s, \langle a, m \rangle) \approx Q_{\text{MDP}}(s,a)$, in which case \cref{eq:measureRegret} would likely give an overestimation of $\mv$, leading to more measurements than required.
\begin{figure}[tb]
\centering
\includegraphics[width=0.75\columnwidth]{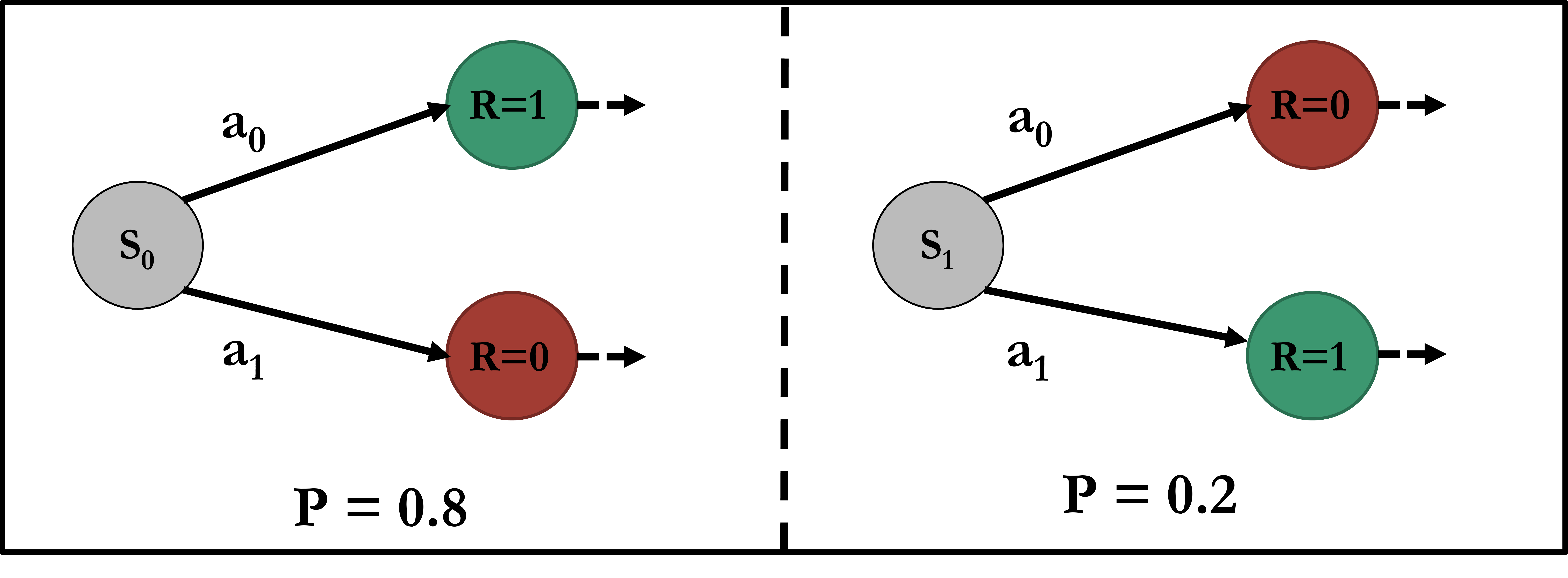}
\caption{An example of a simple belief state. }
\label{fig:MeasurementRegretExample}
\end{figure}

\subsection{Performance Regret of ATM}

Now that $\pi_{\text{ATM}}$ is fully defined, we are interested in its performance loss as compared to an optimal policy $\pi^*$ not restricted by \cref{eq:QApproxATM}.
We first prove the following lemma:
\begin{lemma}
\label{thm:MVOptimal}
    Given a fully known ACNO-MDP $\mathcal{M}$.
    Define $\pi_{\text{ATM}}$ as in \cref{eq:piATM}, and $\pi'_{\text{ATM}}$ as:
$ 
        \pi'_{\text{ATM}}(b) = \langle \max_{a \in A} Q(b,a), \psi(b) \rangle,
$
    with $\psi: b \rightarrow m$. For any choice of $\psi$, the following holds:
    \begin{equation}
    \label{eq:MVOptimal}
    V(\pi_{\text{ATM}}, \mathcal{M}) \geq V(\pi'_{\text{ATM}}, \mathcal{M})
    \end{equation}
\end{lemma}
Intuitively, this lemma states that $m_{\text{MV}}$ is the optimal way of deciding $m$ when following the ATM heuristic.
A full proof is given in \cref{sec:proofs}.
Using this lemma, we can find an upper bound for the performance loss of $\pi_{\text{ATM}}$:
\begin{theorem}
\label{thrm:boundExists}
Given a fully known ACNO-MDP $\mathcal{M}$ with an optimal policy $\pi^*$. The performance loss for the policy following the act-then-measure heuristic $\pi_{\text{ATM}}$ (\cref{eq:piATM}) has the following \emph{minimal upper bound}:
\begin{equation}
\label{eq:th1_bound}
    V(\pi^*, \mathcal{M}) - V(\pi_{\text{ATM}}, \mathcal{M}) \leq \sum_t \gamma^t c
\end{equation}

\end{theorem}

\begin{proof}
We start by proving that \cref{eq:th1_bound} is indeed an  upper bound.
For this, we introduce $\mathcal{M}_{0}$, an ACNO-MDP with the same dynamics and reward function as $\mathcal{M}$, but with $c{=}0$.
In $\mathcal{M}_{0}$, always measuring and taking control actions in accordance to $Q_{\text{MPD}}$ is an optimal policy.
Let $\pi_{\text{Measure}}$ be that policy, than the following holds:
\begin{equation}
\label{eq:proofNoCost}
    V(\pi_{\text{Measure}}, \mathcal{M}_{0}) = V(\pi^*, \mathcal{M}_{0}).
\end{equation}
Since the behaviour of $\pi_{\text{Measure}}$ is indepent of $c$, we can easily relate the expected return of this policy in $\mathcal{M}_0$ to that in $\mathcal{M}$:
\begin{equation}
\label{eq:proofMeasureCost}
     V(\pi_{\text{Measure}}, \mathcal{M}) =  V(\pi_{\text{Measure}}, \mathcal{M}_{0}) - \sum_t \gamma^t c
\end{equation}
Furthermore, we notice $\pi_{\text{Measure}}$ follows the control actions given by $\max_{a \in A} Q(b,a)$.
Thus, via \cref{thm:MVOptimal}:
\begin{equation}
\label{eq:proofATMOpt}
      V(\pi_{\text{ATM}}, \mathcal{M}) \geq  V(\pi_{\text{Measure}}, \mathcal{M})
\end{equation}
Lastly, we note that for a given policy, the expected return in $\mathcal{M}_0$ can never be lower than that in $\mathcal{M}$. Then, in particular:
\begin{equation}
\label{eq:proofOptimalCost}
     V(\pi^*, \mathcal{M}) \leq  V(\pi^*, \mathcal{M}_0)
\end{equation}
Substituting \cref{eq:proofMeasureCost,eq:proofOptimalCost} into \cref{eq:proofNoCost}, then substituting $\pi^*_{\text{ATM}}$ for $\pi_{\text{Measure}}$ following \cref{eq:proofATMOpt}, we find exactly our upper bound.

\begin{figure}[t]
\centering
\includegraphics[width=0.6\columnwidth]{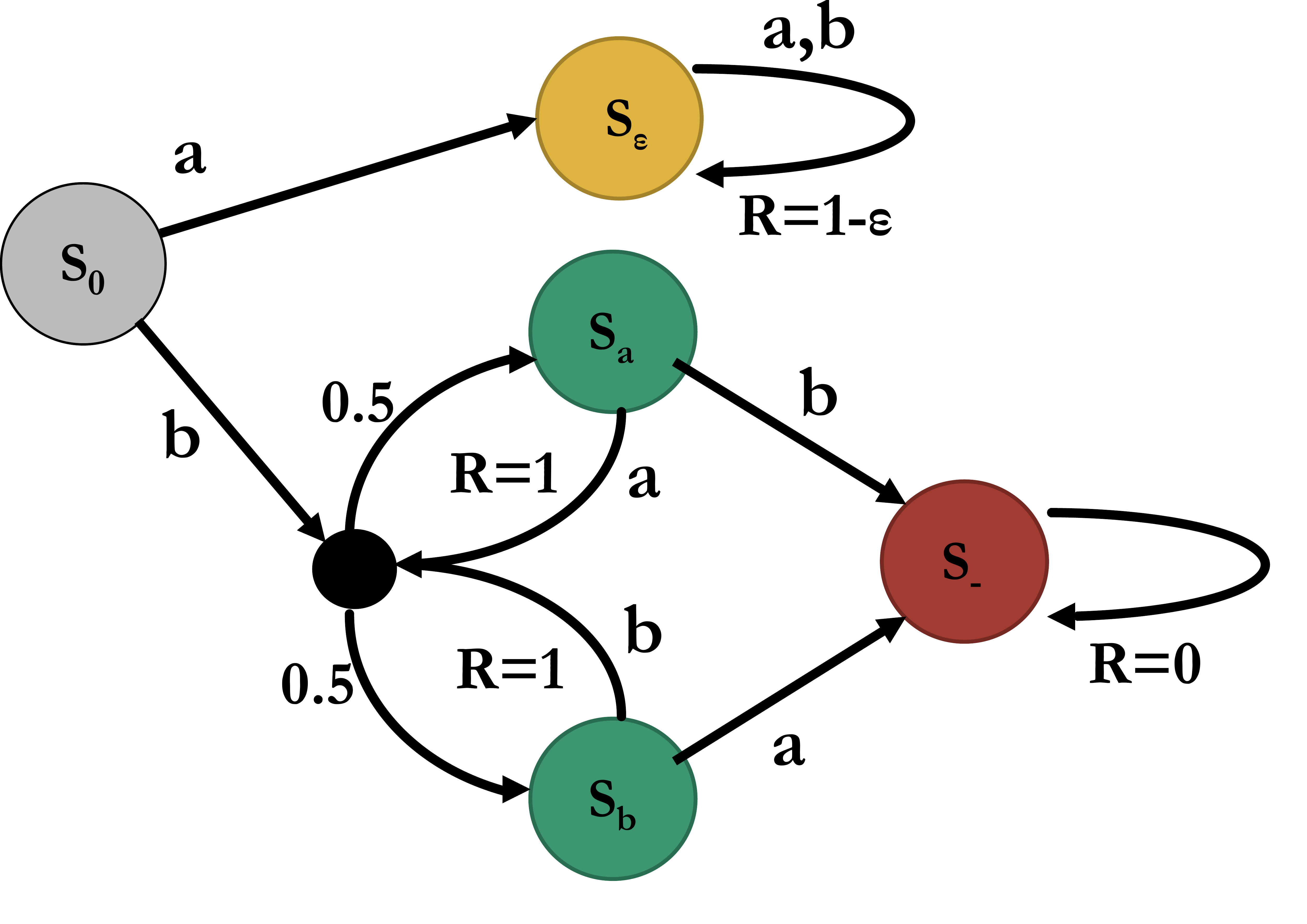} 
\caption{An example where the act-then-measure heuristic can fail for ACNO-MPDs. We assume $c \in [0,0.5]$ and $\epsilon$ is infinitesimally small.}
\label{fig:HeuristicFailExample}
\end{figure}

To prove the given bound is minimal, it suffices to show an ACNO-MDP where the bound is exact, which means no lower bound can exist.
Such an ACNO-MDP is shown in \cref{fig:HeuristicFailExample}.
Using the ATM heuristic, taking action b in state $s_0$ is optimal since both $s_a$ and $s_b$ yield an (infinitesimally) higher expected return than $s_{\epsilon}$ given full state information.
However, after this action the optimal policy would be to measure every step, leading to a lost return of $\sum_t \gamma^t c$.
\end{proof}

In practice, the performance loss of using the ATM heuristic depends on the environment under consideration.
We note the ATM assumption holds in deterministic environments with a single initial state, and has limited impact in environments where $c$ is small relative to the episodic reward.
In contrast, we recall that the AMRL-Q approach does not converge to policies that actively gather information.
This means its performance loss with respect to the baseline policy is unbounded, even when $c$ is small.
\emph{Observe-before-planning} does always converge to $\pi^*$, but in practice may be computationally intractable.

\section{\algName: an ATM-based RL Algorithm for ACNO-MDPs}
\label{sec:algorithm}
To test both the ATM heuristic and measuring value, we implement \emph{dynamic act-then-measure Q-learning} (\algName), an RL algorithm specifically designed for ACNO-MDPs.
A high-level version of the learning loop for an episode is given by \cref{alg:ATMQ-Global}.
The complete pseudo-code is given in \cref{sec:pseudocode}, and a more detailed explanation of all parts of the algorithm is given here.

\paragraph{Belief states.}
To deal with partially unknown states, we implement \textit{discretized belief states} $b_t$, with $b_t(s)$ the estimated probability of being in state $s$ at time $t$. After measuring, belief states are deterministic, i.e. 
\begin{equation}
    b_{t+1}(s) =
    \begin{cases}
        1 & \text{if } s=s_{t+1} \\
        0 & \text{otherwise.}
    \end{cases}
\label{eq:next_belief}
\end{equation}
After a non-measuring action, we instead sample a new belief state, with $ b_{t+1}(s) \sim \sum_{s' \in S} b_t(s') P(s|s',a)$.

\begin{algorithm}[tbp]
\caption{\algName }
\label{alg:ATMQ-Global}
\begin{algorithmic}
\State Initialise transition model $P$, value function $Q$, belief state $b_0$;
\While{episode not completed}
    \State Choose control action $a_t$ according to $Q$ \Comment{\cref{eq:QBelief}}
    \State Choose $m_t$ according to $a_t$  \Comment{\cref{eq:measureCondition}}
    \State Execute $\tilde{a_t} = \langle a_t, m_t \rangle$
    \State Receive reward $r_t$ and observation $o_t$
    \State Determine next belief state $b_{t+1}$ \Comment{\cref{eq:next_belief}}
    \State Update $P$ according to $o_t$ \Comment{\cref{eq:PUpdate,eq:compute_p}}
    \State Update $Q$ according to $r_t$ and $P$ \Comment{\cref{eq:QUpdate}}
    \State Update $Q$ using model-based training 
\EndWhile
\State \textbf{return } $\sum_t \gamma^t r_t$
\end{algorithmic}
\end{algorithm}

\paragraph{Transition model.}
To estimate our transition probabilities, we apply the \textit{Bayesian MDP} approach as introduced by \citet{DBLP:conf/uai/DeardenFA99}.
In this framework, a transition function $P(\cdot\mid s,a)$ is given by a \textit{Dirichlet distribution} $D(s,a)$, as parameterised by $\vec{\alpha} = \{\alpha_{s,a,s_0}, \alpha_{s,a,s_1}, ...\}$.
In the standard MDP-setting, $\alpha_{s,a,s'}$ is given by a (uniform) prior, plus the number of times a transition has already occurred.
For the ACNO-MDP setting, we change this to the number of times it has been \textit{measured}.
Thus, at every step we update our model as follows:
\begin{equation}
\label{eq:PUpdate}
    \alpha_{s,a,s'} \leftarrow 
    \begin{cases}
    \alpha_{s,a,s'}+1 & \text{if } a_{t-1} = a,m_t=1,  \\ & b_t(s) = 1, b_{t+1}(s') = 1;\\
    \alpha_{s,a,s'} & \text{otherwise,}
    \end{cases}
\end{equation}
and define estimated transition probabilities as:
\begin{equation}
    P(s' \mid s,a) = \E\left[s' \mid D(s,a)\right] = \frac{ \alpha_{s,a,s'}}{\alpha_{s,a} },
\label{eq:compute_p}
\end{equation}
where $\alpha_{s,a} = \sum_{s' \in S} \alpha_{s,a,s'}.$

\paragraph{Value function.}
To estimate the values of belief states, we make use of the \textit{replicated Q-learning method}, as introduced in \citet{DBLP:conf/aaai/Chrisman92} and formalized by \citet{DBLP:conf/icml/LittmanCK95}.
In this method, we assume the optimal action for any belief state can be given as a linear function over all states.
With this assumption, we choose a control action $a$ in belief state $b$ as follows:
\begin{equation}
\label{eq:QBelief}
    a_t = \max_{a \in A} Q(b_t,a) = \max_{a \in A}\sum_{s \in S} b_t(s) Q(s,a).
\end{equation}
To update the Q-values, we use the following update rule:
\begin{equation}
\label{eq:QUpdate}
    Q(s,a) \leftarrow (1-\eta_s) Q(s,a) + \eta_s (\tilde{r} + \gamma \Psi(s,a) ),
\end{equation}
with $\eta_s = b(s)\eta$ the weighted learning rate and $\Psi(s,a)$ the estimated future return after state-action pair $(s,a)$:

\begin{equation}
\label{eq:futureReturn}
    \Psi (s,a) = \sum_{s' \in S} P(s' \mid s, a) \max_{a'} Q(s',a').
\end{equation}
Lastly, to incentivize exploration, we create an \textit{optimisitc} variant of $Q$.
For this, we define an exploration bonus $\rb$:
\begin{equation}
    \rb(s,a) = \max \left[0, \frac{N_{\text{opt}} - \alpha_{s,a}}{N_{\text{opt}}}(R_{\text{max}}-Q(s,a))\right],
\end{equation}
    with $R_{\text{max}}$ the maximum reward in the ACNO-MDP and $N_{\text{opt}}$ a user-set hyperparameter.
We use this metric to create an \textit{optimistic value function} $Q_{\text{opt}}$:
\begin{equation}
    Q_{opt}(s,a) = Q(s,a) + \rb(s,a),
\end{equation}
which we use instead of the real Q-value in \cref{eq:QBelief,eq:futureReturn}.
Inspired by R-Max \citep{DBLP:journals/jmlr/BrafmanT02}, our metric initially biases all $Q$-values such that $Q(s, a) = R_{\text{max}}$, and removes this bias in a number of steps.
However, instead of a binary change, $\rb$ makes this transition in $N_{\text{opt}}$ (linear) steps.
In practice, we found this gives a stronger incentive to explore all state-action pairs more uniformly, leading to a faster convergence rate.

\paragraph{Measurement condition.}
In an RL setting, we note there are two distinct reasons for wanting to measure your environment: \textit{exploratory measurements} to improve the accuracy of the model, and \textit{exploitative measurements} which improve the expected return.
For the latter, we have already introduced \emph{measuring value} ($\mv$) in \cref{sec:MeasurmentRegret}.

For the former, we again draw inspiration from R-Max \citep{DBLP:journals/jmlr/BrafmanT02} by introducing a parameter $N_m$, and measure the first $N_m$ times a state-action pair is visited.
We keep track of this number using $\vec{\alpha}$ as specified in \cref{eq:PUpdate}.
Lastly, we specify to take exploratory measurements only if we are certain about the current state, since no model update is performed otherwise (\cref{eq:PUpdate}).

Combining both types of measurements, we construct the following condition for deciding when to measure:
\begin{equation}
\label{eq:measureCondition}
    m_t {=} \begin{cases}
        1 & \text{if } \exists s: b_t(s) {=} 1 \land \alpha_{s,a_t } {<} N_m; \\
        m_{\mv}(b_t, a_t) & \text{otherwise}.
    \end{cases}
\end{equation}  

\paragraph{Model-based training.}
Lastly, inspired by the Dyna-framework \citep{DBLP:journals/sigart/Sutton91}, at each step we perform an extra $N_{\text{train}}$ \textit{training steps}.
For this, we pick a random state $s$ and action $a$, create a \emph{simulated reward}, and use this to perform a Q-update (\cref{eq:QUpdate}).
For this simulated reward, we use the average reward received thus far $R_{s,a}$, which we initialise as 0 and update each step:
\begin{equation}
\label{eq:RUpdate}
    R_{s,a} {\leftarrow}
    \begin{cases}
    \frac{R_{s,a}\cdot\alpha_{s,a}+r_t}{ \alpha_{s,a}+1} & \text{if } a_{t-1} {=} a,  m_t{=}1, b_t(s) {=} 1;\\
    R_{s,a} & \text{otherwise.}
    \end{cases}
\end{equation}

Although originally proposed to deal with changing environments, we mainly use the Dyna approach to speed up the convergence of the Q-table.
This is especially relevant for our setting, where even the Q-values for actions never chosen by our policy need to be accurate to estimate $\mv(b_t,a_t)$.

\section{Empirical Evaluation}
\label{sec:eval}

In this section, we report on our empirical evaluation of \algName in a number of environments. 
We first give a description of the setup of both the algorithms and environments.
Then, we show the results of our experiments, and lastly, we highlight some key conclusions.
All used code and data can be found at \gitLink.

\begin{figure}[tb]
     \centering
     \includegraphics[width=0.52\columnwidth]{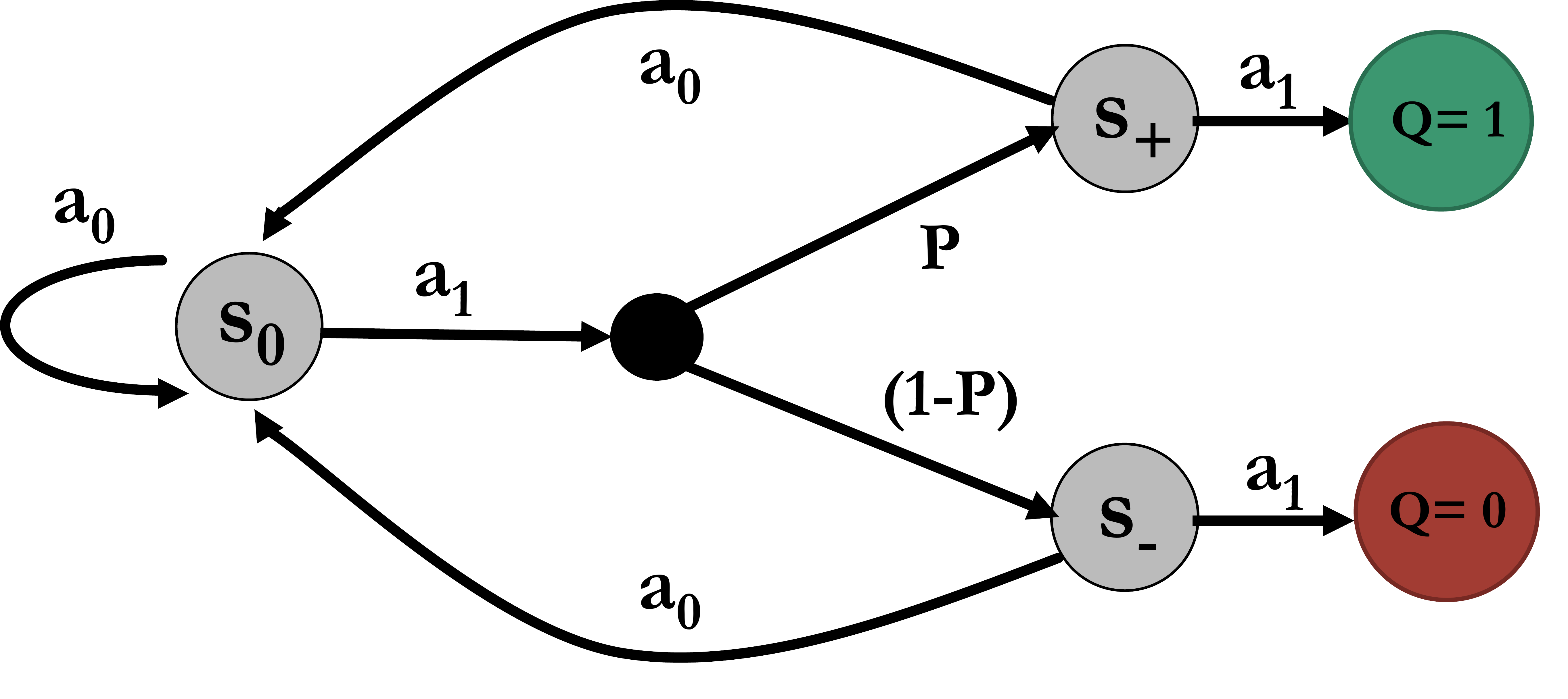}
     \caption{The \textit{measuring value} environment used to test if an agent can determine the value of measuring.}
     \label{fig:regretEnv}
\end{figure}

\subsection{Experimental Setup}
\label{sec:setup}


We test the following algorithms:

\emph{Dyna-ATMQ:}
We implement \algName as described in \cref{sec:algorithm}.
We set $\gamma = 0.95$, $\eta = 0.1$, $N_b = 100$ and $N_{\text{opt}} = N_m = 20$.
For offline training, we choose random states and update their current optimal action with probability $\epsilon_{\text{train}} = 0.5$, and a random different action otherwise.
We use $N_{\text{train}}=25$, but also test a non-dynamic variant with $N_{\text{train}}=0$, which we'll refer to as \algNameShort.

\emph{AMRL-Q:}
For AMRL-Q, we re-implement the algorithm as specified in \citet{DBLP:conf/ai/BellingerC0T21}.
We set $\gamma=0.95$ and $\alpha=0.1$ to match those of \algName, and use initial measurement bias $\beta=0.1$ as described in the paper.
Lastly, we use $\epsilon=0.1$ for the first $90\%$ of all episodes but switch to a fully greedy approach for the last $10\%$.

\emph{ACNO-OTP:}
We implement the \emph{observe-before-planning} algorithm specified in \citet{DBLP:conf/nips/NamFB21}, using an altered version of the original code, which we refer to as ACNO-OTP.
We explain the changes to the original code in more detail in \cref{sec:ACNOChanges}.
For the experiments, we use $\gamma = 0.95$ and \textit{ucb-coefficient} $c=10$.
We perform 25.000 rollouts per step at a max search depth of 25, with between 1800 and 2000 particles.
Since we are interested in results after convergence, we limit the exploitation phase to the last 50 episodes and only compare these last episodes.

\begin{table*}[tb]
\centering
\begin{minipage}{.495\textwidth}\centering
\resizetable{
\begin{tabular}{@{}lcrrcrrcrr@{}}
\toprule
                   && \multicolumn{8}{c}{Measurement   Cost}                                                       \\ \cmidrule{3-10}
                   && \multicolumn{2}{c}{0.05}    && \multicolumn{2}{c}{0.10}    && \multicolumn{2}{c}{0.20}       \\ \cmidrule{3-4} \cmidrule{6-7} \cmidrule{9-10} 
Algorithm          && SR    & M                   && SR    & M                   && SR    & M                      \\ \cmidrule{1-1}
\algNameShort      && 0.94  & 1.30                && 0.76  & 0.50                && 0.78  & 0.11                   \\
\algName           && 0.93  & 1.34                && 0.86  & 1.14                && 0.82  & 0.16                   \\
AMRL               && 0.82  & 0.00                && 0.80  & 0.00                && 0.78  & 0.00                   \\
ACNO-OTP           && 0.94  & 1.18                && 0.81  & 0.00                && 0.79  & 0.00                   \\
\bottomrule
\end{tabular}}
\end{minipage}
\begin{minipage}{.495\textwidth}\centering
\resizetable{
\begin{tabular}{@{}lcrrcrrcrr@{}}
\toprule

                   && \multicolumn{8}{c}{Variant}                                                                           \\ \cmidrule{3-10}
                   && \multicolumn{2}{c}{Deterministic}    && \multicolumn{2}{c}{Semi-slippery}    && \multicolumn{2}{c}{Slippery}     \\ \cmidrule{3-4} \cmidrule{6-7} \cmidrule{9-10}
Algorithm          && SR    & M                            && SR    & M                            && SR    & M                        \\ \cmidrule{1-1}
\algNameShort      && 1.00  & 0.00                         && 0.75  & 2.65                         && 0.02  & 0.76                     \\
\algName           && 1.00  & 0.00                         && 0.65  & 2.95                         && 0.03  & 0.93                     \\
AMRL               && 1.00  & 0.00                         && 0.41  & 0.00                         && 0.03  & 0.00                     \\
ACNO-OTP           && 1.00  & 0.00                         && 0.40  & 0.00                         && 0.04 & 0.00                     \\
\bottomrule
\end{tabular}}
\end{minipage}
\caption{Average scalarized return (SR) and the number of measurements (M) after training, in the measuring value (left) and frozen lake (right) environments.
Results are gathered over 5 repetitions, and present the average over the last 50 episodes.}
\label{tab:Results}
\end{table*}

For our testing, we use the following environments:

\emph{Measuring value:}
As a simple environment to test measuring value, we convert our example from \cref{fig:MeasurementRegretExample} to a graph, as shown in \cref{fig:regretEnv}.
This environment consist of three state $S=\{s_0, s_+, s_-\}$, with $s_0$ as the initial state.
Our agent can choose actions from action space $A=\{a_0,a_1\}$, where $a_0$ always returns the agent to the initial state.
From state $s_0$, taking action $a_1$ results in a transition to $s_+$ with probability $p$ and a transition to $s_-$ with probability $p-1$.
Taking action $a_1$ in the states $s_+$ and $s_-$ ends the episode and returns rewards $r=1$ and $r=0$, respectively.

For this environment, we can explicitly describe its optimal strategy and its expected value.
We notice that depending on $p$ and $c$, such strategies either try to measure the (otherwise indistinguishable) states $s_+$ and $s_-$, or they do not.
When not measuring, our expected return is always $p$.
When measuring, our expected return in $s_+$ is $1-c$, and in $s_-$ it is the expected return of $s_0$ minus $c$.
Combining this, we can calculate the expected return for $s_0$ with a measuring policy:
\begin{equation}
    \E_\pi \left[ \sum_t \gamma^t \tilde{r}_t \right] = \sum_{n=0} \gamma^{2n} \Big( p\cdot \big(1-p\big)^n \big(1-c(n+1)\big)  \Big),
\end{equation}
where $n$ is the number of measurements required before the episode ends.
For our experiments, we set $\gamma = 1$ and $p=0.8$, which means measuring is profitable for $c\leq 0.16$.

\emph{Frozen lake:}
As a more complex toy environment, we use the standard \textit{openAI gym} frozen lake environment \citep{brockman2016openai}, which describes an $n \times n$ grid with a number of `holes'. 
The goal of the agent is to walk from its initial state to some goal state without landing on any hole spaces.
The agent receives a reward $r=1$ if it reaches the goal and $r=0$ otherwise.
The episode ends once the agent reaches the goal state or a hole tile.
In our testing, we will use the pre-defined $4\times 4$ and $8\times8$ map settings, as well as larger maps randomly generated, all with a measuring cost $c=0.05$.
The agent has action space $A=\{\mathrm{left}, \mathrm{down}, \mathrm{right}, \mathrm{up}\}$, but we consider three variations of their interpretation. 
Firstly, we use both the predefined deterministic and non-deterministic (or \textit{slippery}) settings from the standard gym.
In the deterministic case, the agent is always moved in the given direction, while in the slippery case it has an equal probability to move in the given or a perpendicular direction.
We also implement and test a more predictable \textit{semi-slippery} variant, where the agent always moves in the given direction, but has a $0.5$ chance of moving two spaces instead of one.

\begin{figure}[tb]
\centering
\includegraphics[width=0.95\columnwidth]{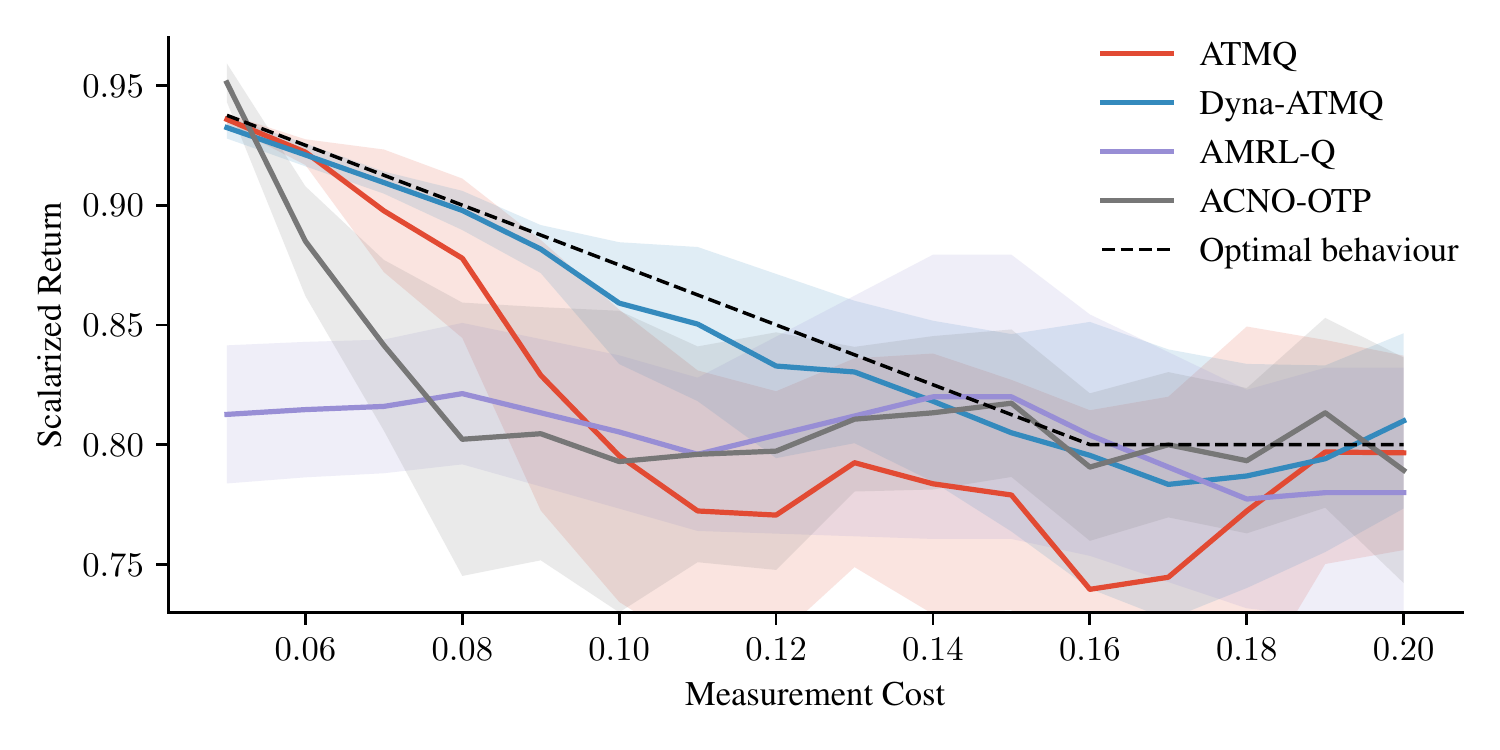}
\includegraphics[width=0.95\columnwidth]{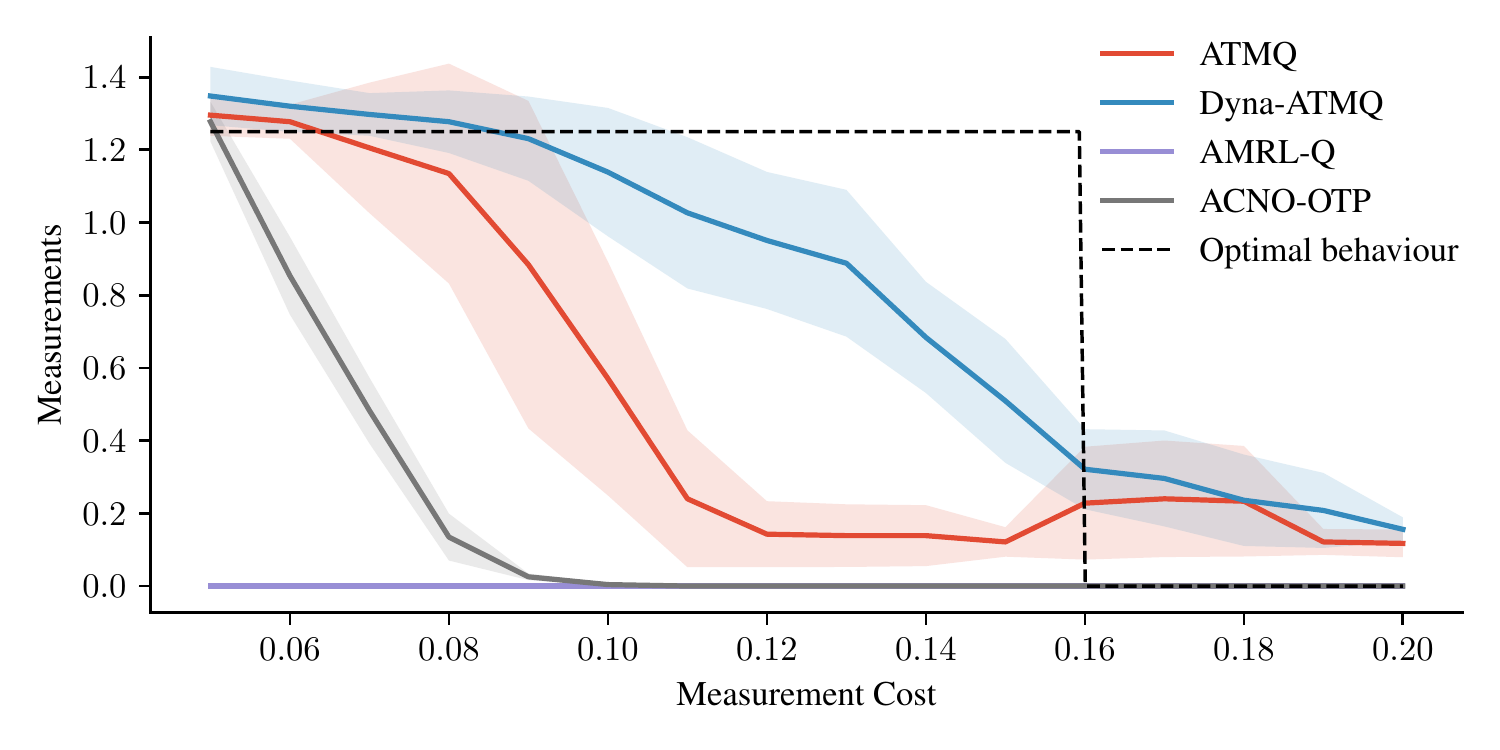}
\caption{Scalarize returns and the number of measures in the measuring value environment, with $p=0.8$ and varying measurement costs. Values are averages over 5 repetitions after convergence.}
\label{fig:RegretResults}
\end{figure}

\begin{figure}[tb]
\centering
\includegraphics[width=0.95\columnwidth]{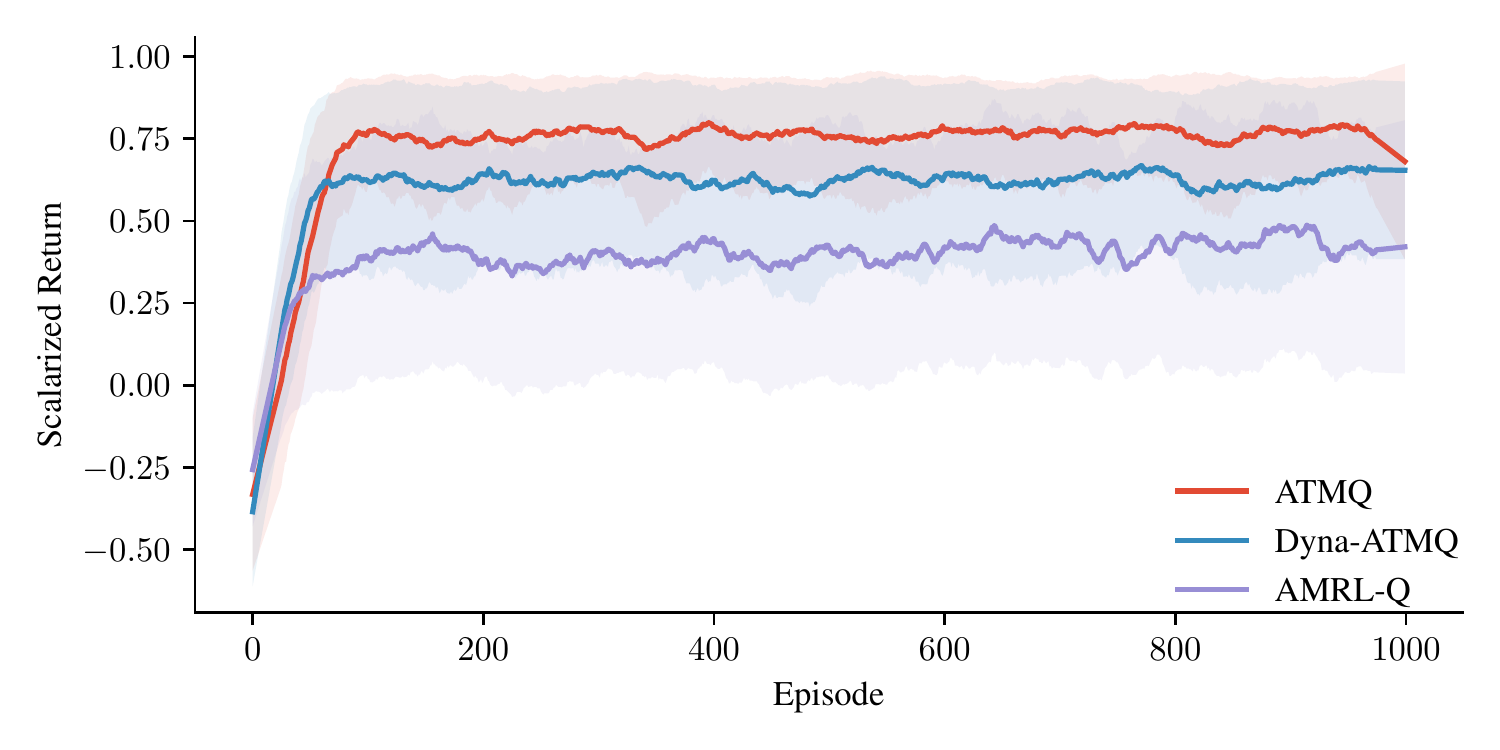}
\includegraphics[width=0.95\columnwidth]{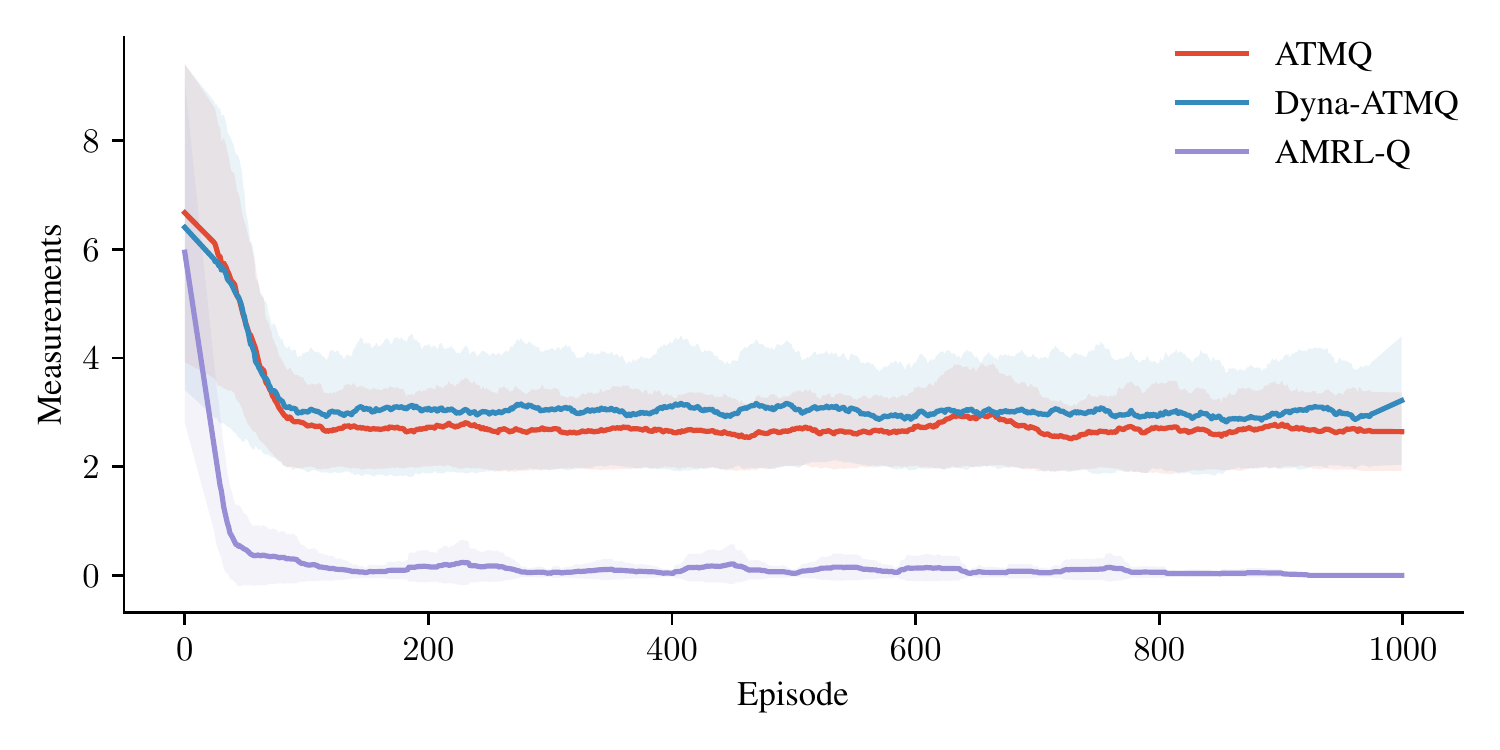}
\caption{Empirical results on semi-slippery $4\times4$ frozen lake environment, gathered over 5 repetitions.}
\label{fig:SmallLakeResults}
\end{figure}

\subsection{Experimental Results}
\label{sec:results}

To test the \emph{measuring value} metric, we run \algName on the measuring value environment for a range of different measurement costs.
The results can be found in \cref{tab:Results} (left) and \cref{fig:RegretResults}.
We notice that both \algName variants, as well as ACNO-OTP, can find close-to-optimal measuring and non-measuring policies.
However, as clearly seen in \cref{fig:RegretResults} (bottom), all algorithms use non-measuring policies for costs where measuring would still be optimal.
The Dyna-variant of \algNameShort performs slightly better than both others, but the difference is minimal, especially in terms of rewards.
In contrast, AMRL-Q always converges to a non-measuring policy, regardless of measurement cost.

To test how the \textit{act-then-measure}-heuristic effects performance for varying amounts of non-determinism, we run tests on all three variants of the $4\times 4$ frozen lake environment.
Results are given in \cref{tab:Results} (right).
For both the deterministic and slippery variants, both versions of \algNameShort perform about on par with both of its predecessors.
For the former, it converges to an optimal non-measuring policy, and for the latter none of the algorithms get a significantly positive result.
However, in the semi-slippery environment, both variants significantly outperform both ACNO-OTP and AMRL-Q, with the non-training variant performing slightly better.
To visualize, training curves for our algorithm and AMRL-Q in this environment are shown in \cref{fig:SmallLakeResults}.

\begin{figure}[t]
    \centering
    \includegraphics[width=0.95\columnwidth]{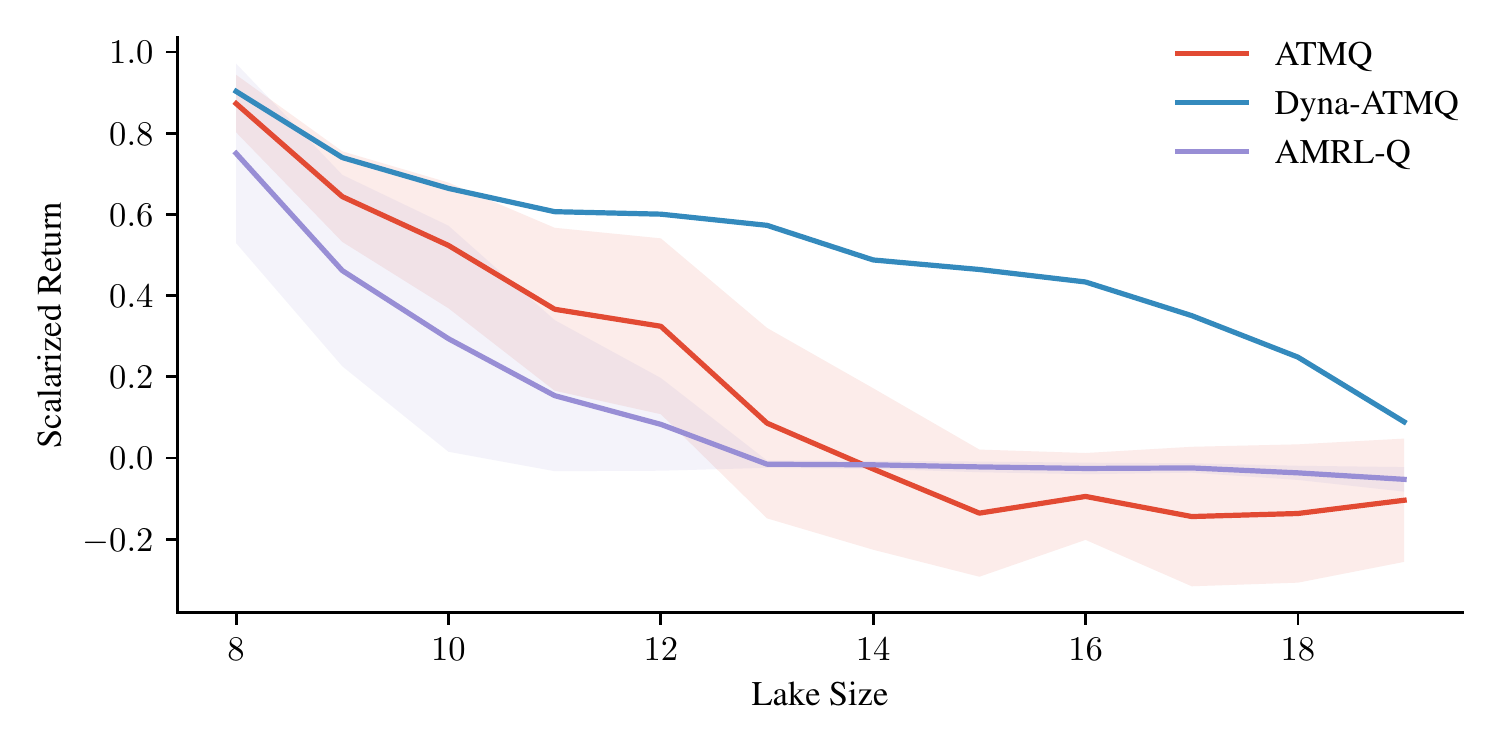}
    \caption{Average scalarized return (after convergence) for semi-slippery frozen lake environment, for different sizes. Results for \algNameShort and AMRL-Q averaged over 5 repetitions, for \algName over 1.}
    \label{fig:LargeLakeResults}
\end{figure}

To test the scalability of algorithms using the \textit{act-then-measure}-heuristic, we test the performance of \algName on a number or larger semi-slippery frozen lake Environments.
Results of both \algNameShort variants and AMRL-Q are shown in \cref{fig:LargeLakeResults}
\footnote{Because of high computation times, we were unable to obtain results for ACNO-OTP in these environments.}
.
Although the performance of both variants drops quickly with the size of the environment, they are able to achieve above-zero returns for far bigger environments than AMRL-Q.
The Dyna-variant performs better for larger environments, even after convergence.

\subsection{Discussion}
\label{sec:discussion}
Based on our results, we make the following claims:

\paragraph{Measuring value is a suitable metric.}
In \cref{tab:Results}(right), we notice \algName converges to a non-measuring policy in the deterministic environments, as expected.
For stochastic environments, we note it makes more measurements than our baselines but gets better or equal returns.
This suggests it correctly identifies when taking measurements is valuable.
We notice suboptimal measuring behaviour only when the difference in return between measuring and non-measuring is small, but note that this could be caused by slight errors in our Q-table.

\paragraph{\algName performs well in small environments.}
In both the measuring value and small frozen lake environments, we find \algName performs better than the bound given by \cref{thrm:boundExists}.
Moreover, it outperforms or equals all baseline algorithms while staying computationally tractable.

\paragraph{\algName is more scalable than current methods.}
\algName stays computationally tractable for larger environments than ACNO-OTP, while yielding higher returns than AMRL-Q.
More generally, we note that our current implementation of the ATM-heuristic approximates the Q-values of states in a way that is known to lead to errors for highly uncertain settings \citep{DBLP:conf/icml/LittmanCK95}.
This suggests a more sophisticated algorithm using the ATM heuristic could improve scalability.

\section{Related Work}
\label{sec:relatedWork}
For the tabular ACNO-MDP setting, three RL algorithms already exist: the AMRL-Q \citep{DBLP:conf/ai/BellingerC0T21}, and the \textit{observe before planning} and the ACNO-POMCP algorithms \citep{DBLP:conf/nips/NamFB21}.
The latter is shown to perform worse than \textit{observe before planning} so is not considered in this paper, the other two are discussed in detail in \cref{sec:setting} and used as baselines in our experiments.
As far as we know, there are no other works with which we can directly compare our results.

Another closely related work is that of \citet{DBLP:journals/ai/Doshi-VelezPR12}.
They introduce a framework in which agents explore a POMDP, but have the additional option to make `action queries' to an oracle.
The method used is comparable to ours and their concept of \emph{Bayesian Risk} resembles the concept of measuring value introduced here.
However, since their method relies on action queries instead of measurements, results cannot easily be compared.

We also note some related papers which explore active measure learning in different contexts.
\citet{DBLP:journals/corr/abs-2011-00825} propose a method for AMRL which relies on a pre-trained neural network to infer missing information.
\citet{DBLP:conf/cdc/GhasemiT19} propose a method to choose near-optimal measurements on a limited budget per step, which can be used to improve pre-computed `standard' POMDP policies.
\citet{DBLP:conf/miccai/BernardinoJLCSC22} investigate diagnosing patients using an MDP approach, in which the action themselves correspond to taking measurements.
\citet{DBLP:conf/nips/MateKXPT20} consider a restless multi-armed bandit setting where taking an action simultaneously resolves uncertainty for the chosen arms.
Lastly, \citet{DBLP:conf/ewrl/Araya-LopezBTC11} study how to approximate an MDP without a reward function.

\section{Conclusion}
\label{sec:conclusion}
In this paper, we proposed the \emph{act-then-measure} heuristic for ACNO-MDPs and proved that the lost return for following it is bounded.
We then proposed \emph{measuring value} as a metric for the value of measuring in ACNO-MDPs.
We describe \algName as an RL algorithm following the ATM heuristic, and show empirically it outperforms prior RL methods for ACNO-MDPs in the tested environments.

Future work could focus on improving the performance of \algName, for example by implemeting more sophisticated action choices and Q-updates, or by taking taking epistemic uncertainty more into account for exploration.
To improve scalability, an interesting line of research is to adapt an already existing method to use the ATM-heuristic.
Model-based methods, such as MBPO \citep{DBLP:conf/nips/JannerFZL19}, are most suitable for such adaptations.
Another possible direction is to investigate the ATM-heuristic in  the more general active measure POMDP setting, in which we lose the assumption of complete and noiseless measurements.
Lastly, our approach could be considered in different multiobjective settings, such as one where the preference function for reward and measurement cost is not known a-priori \citep{MOO}, or where the measuring cost is used as a constraint \citep{DBLP:conf/cdc/GhasemiT19}.

\nocite{DBLP:conf/ewrl/Araya-LopezBTC11}
\nocite{emami2015pomdpy}

\section*{Acknowledgments}
This research has been partially funded by NWO grant NWA.1160.18.238 (PrimaVera) and the ERC Starting Grant 101077178 (DEUCE).

\bibliography{main}

\appendix

\newpage
\newpage
\section{Pseudo-code \algName}
\label{sec:pseudocode}


\begin{algorithm}[H]
\caption{\textsc{BAM-QMDP}(episodes)}
\label{alg:bamqmdp_run}
\begin{algorithmic}
\For{$s,s' \in S $ and $a \in A$}
    \State Set $\alpha_{s,a,s'} = 1/|S|, N_Q(s,a) = 0 $
    \State Set $Q(s,a) = 0$, $Q_{opt}(s,a) = 1, R(s,a) = 0$
\EndFor
\State Set $r_{total} = 0$
\For{$i < $episodes}
    \State $r_{eps}, Q, Q_{opt}, \vec{\alpha}, R $ 
    \State \hspace{15pt} $\leftarrow$ \textsc{RunEpisode}($Q, Q_{opt}, \vec{\alpha}, R, b_0$)
    \State $r_{total} \leftarrow r_{total} + r_{eps}$
\EndFor
\State $\mathbf{return } r_{total}$
\end{algorithmic}
\end{algorithm}

\begin{algorithm}[H]
\caption{\textsc{RunEpisode} ($Q, Q_{opt},\vec{\alpha},R, b_0$)}
\label{alg:bamqmdp}
\begin{algorithmic}
\State $b \leftarrow b_0$
\State $r_{episode} = 0$
\While{episode not done}
    \State $a \xleftarrow{} $\textsc{FindGreedyAction}($Q_{\text{opt}},b$)
    \State $b_{\text{next}} \xleftarrow{} $\textsc{SampleNextBelief}($P,b,a$)
    \State $m \xleftarrow{} ( \exists s, b(s) = 1 \land \alpha_{s,a} < N_m ) \lor \mathbf{R}(b_{\text{next}}) <c$
    \State Take action $(a,m) \rightarrow (o,r)$
    \If{$m = 1$}
        \State $P,R, \vec{\alpha} \leftarrow $\textsc{UpdateModel}($R, \vec{\alpha}, b,a,o,r$)
        \State $b \leftarrow o$
    \Else
        \State $b \leftarrow b_{\text{next}}$
    \EndIf
    \State $Q, Q_{\text{opt}} \leftarrow $ \textsc{UpdateQ}($P,Q_{\text{}},b,a,r$)
    \State $Q, Q_{\text{opt}} \leftarrow$ \textsc{ModelBasedTraining}($P,Q,R$)
    
    \State $r_{episode} \leftarrow r_{episode}+r$
\EndWhile
\State $\mathbf{return } r_{episode}, Q, Q_{opt}, \vec{\alpha},R$
\end{algorithmic}
\end{algorithm}

\begin{algorithm}[H]
\caption{\textsc{SampleNextBelief}($P,b,a$)}
\label{alg:sample_b}
\begin{algorithmic}
\For{$s \in S$}
    \For{$s' \in S$}
        \State $b_{\text{next,full}}(s') \leftarrow b_{\text{next,full}}(s') + P(b(s),a,s')$
    \EndFor
\EndFor
    \For{$i<N_b$}
        \State $s_{\text{next,i}} \sim b_{\text{next,full}}$
        \State $b_{\text{next}}(s_{\text{next,i}}) \leftarrow b_{\text{next}}(s_{\text{next,i}}) + 1/N_b$
    \EndFor
\State $\mathbf{return } b_{\text{next}}$
\end{algorithmic}
\end{algorithm}

\begin{algorithm}[H]
\caption{\textsc{FindGreedyAction}($Q,b$)}
\label{alg:find_action}
\begin{algorithmic}
\State initialise $Q_b(a) = 0, \forall a \in A$
\For{$ (s,a) \in S\times A$}
        \State $Q_b(a) \leftarrow Q_b(a) + b(s) Q(s,a)$
\EndFor
\State $\mathbf{return } \arg \max_{a\in A}Q_b(a)$
\end{algorithmic}
\end{algorithm}

\begin{algorithm}[H]
\caption{\textsc{UpdateModel}($\vec{\alpha},b,a,o,r$)}
\label{alg:update_P}
\begin{algorithmic}
\If{ $[\exists s \in S | b(s) = 1]$}
    \State $\alpha_{s,a,o} \leftarrow  \alpha_{s,a,o} + 1$
    \State $\alpha_{s,a} \leftarrow \alpha_{s,a} + 1$
    \For{$s' \in S$}
        \State $P(s'\mid s,a) = \alpha_{s,a,s'} / \alpha_{s,a} $
    \EndFor
    \State $R(s,a) \leftarrow \frac{r + (\alpha_{o,a}-1)R(s,a)}{\alpha_{o,a}}$
\EndIf
\State $\mathbf{return } P, \vec{\alpha}, R$
\end{algorithmic}
\end{algorithm}

\begin{algorithm}[H]
\caption{\textsc{UpdateQ}($P,Q,R,b,a,r$)}
\label{alg:update_Q}
\begin{algorithmic}
\For{$s \in S$}
    \State $\eta_{s} \leftarrow b(s)\eta$
    \State $\Psi \leftarrow \sum_{s' \in S}P(s'\mid s,a) \max_{a'}Q(s', a')$
    \State $Q(s,a) \leftarrow [1-\eta_{s}]Q(s,a) + \eta_{s} [r+\gamma \Psi]$
    \State $Q_{\text{opt}}(s,a) \leftarrow Q(s,a) + \text{R-Max*}(s,a) $
\EndFor
\State $\mathbf{return } Q, Q_{\text{opt}}$
\end{algorithmic}
\end{algorithm}

\begin{algorithm}[H]
\caption{\textsc{ModelBasedTraining}($Q, P, R$)}
\label{alg:update_Q_global}
\begin{algorithmic}
\For{$ i < N_{\text{train}}$}
    \State Pick a random $s \in S$
    \For{$s' \in S$}
        \State $b(s') \leftarrow \delta_{s,s'}$
    \EndFor
    \State $a \leftarrow \arg \max_{a} Q(s,a)$ with $p=\frac{1}{2}$, otherwise random
    \State $r \leftarrow R(s,a)$
    \State \textsc{UpdateQ}($P,Q,R,b,a,r$)
    
\EndFor
\State $\mathbf{return } Q, Q_{\text{opt}}$
\end{algorithmic}
\end{algorithm}

\section{Implementation Details for ACNO-OTP}
\label{sec:ACNOChanges}
For ACNO-OTP, we used the publicly available code from \citet{DBLP:conf/nips/NamFB21}, but altered it to work for any openAI environment.
For the exploration phase, we could not rely on the EULER-algorithm \citep{DBLP:conf/icml/ZanetteB19} used in the original, since the frozen lake environment does not follow the requirement that all states are reachable within 4 steps.
Instead, we implement a simple Q-learning method, with as a reward the \emph{state-action counter reward} introduced in \citet{DBLP:conf/ewrl/Araya-LopezBTC11}.
For the exploitation phase, as in the original, we use a publicly available version of POMCP for Python \citep{emami2015pomdpy}.

However, in testing, we notices the negative rewards for measuring actions were leading to drastic under-approximations of the Q-values in the planning phase, in two distinct ways.
Firstly, in the rollout phase of the algorithm, following a random policy in an ACNO-MDP setting is likely to result in many measuring actions and thus negative returns.
This is easily fixed by specifying that $\pi_{\text{rollout}}$ may only contain non-measuring actions, which has no negative impact on the algorithm.
More subtly, back-propagation for some measuring action also has an effect on all actions on the path to it, which might lead to undesirable results.
An example is visualized in \cref{fig:POMCPFail2}, where the Q-value of going from state $s_0$ to state $s_1$ is updated to be $-c$, even though from $s_1$ better paths are present and unexplored.
This problem is somewhat more inherent to POMCP and thus harder to solve.
In our implementation, we included the specification that non-measuring actions could never have negative rewards, which solved the problem for the tested environments.
However, this method is clearly applicable only for environments without negative rewards, and even these might lead to a sub-optimal bias for non-measuring actions.

As a general remark, we note that in history-based planners such as POMCP, the information gained by taking a measurement cannot always be exploited fully.
As visualized in \cref{fig:POMCPFailExample}, the knowledge that measuring and non-measuring actions have the same effect on the environment is not explicit in a history-based representation, even if it is represented in the used model.
Concretely, for POMCP this means that to approximate the value of taking these actions, both need to be sampled separately, which can lead to higher computation times.

\begin{figure}[tb]
    \centering
    \includegraphics[width=0.6\columnwidth]{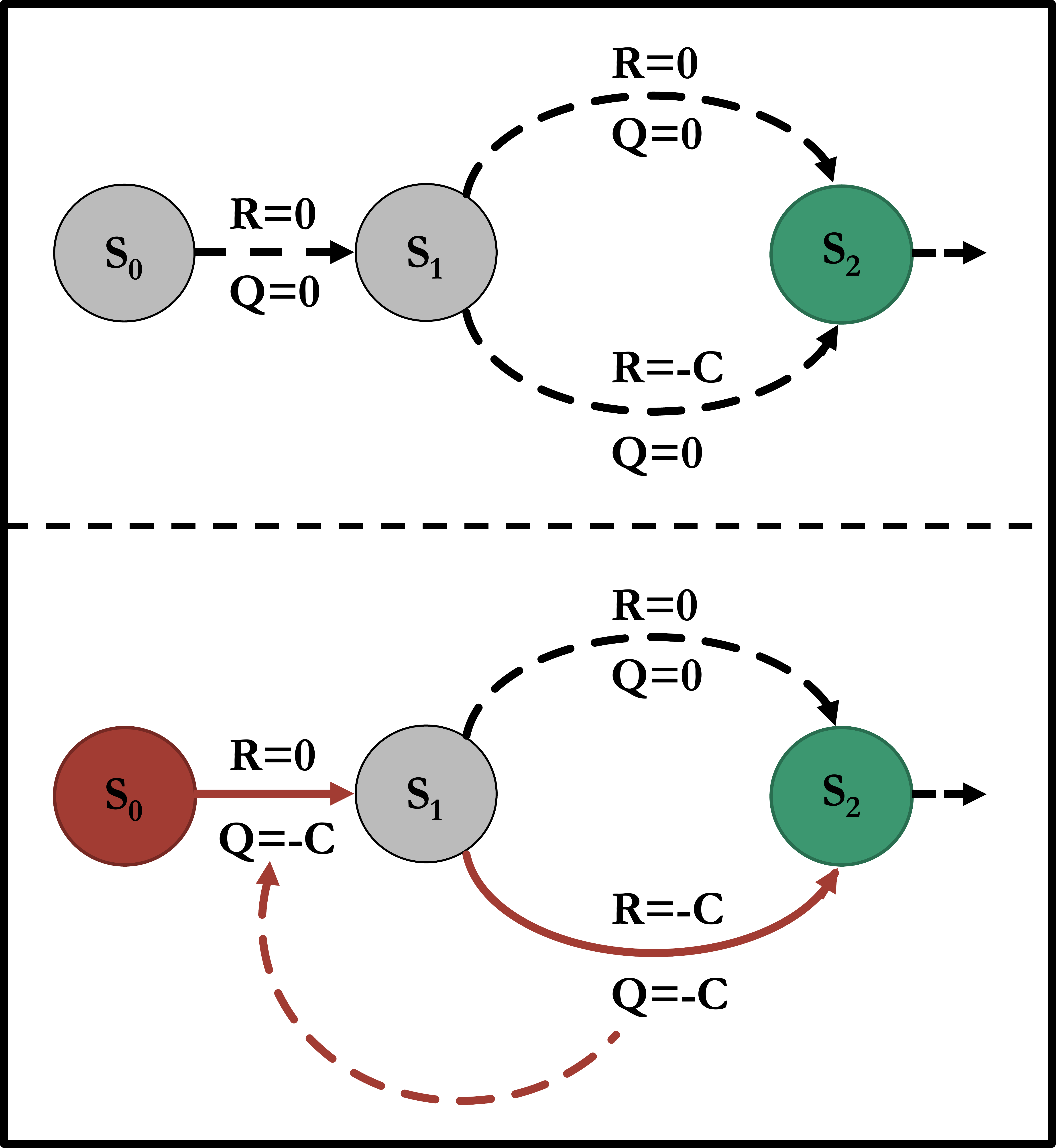}
    \caption{A simplified visualization of an ACNO-MDP where measuring actions lead to undesired effects.}
    \label{fig:POMCPFail2}
\end{figure}

\begin{figure}[t]
\centering
\includegraphics[width=0.9\columnwidth]{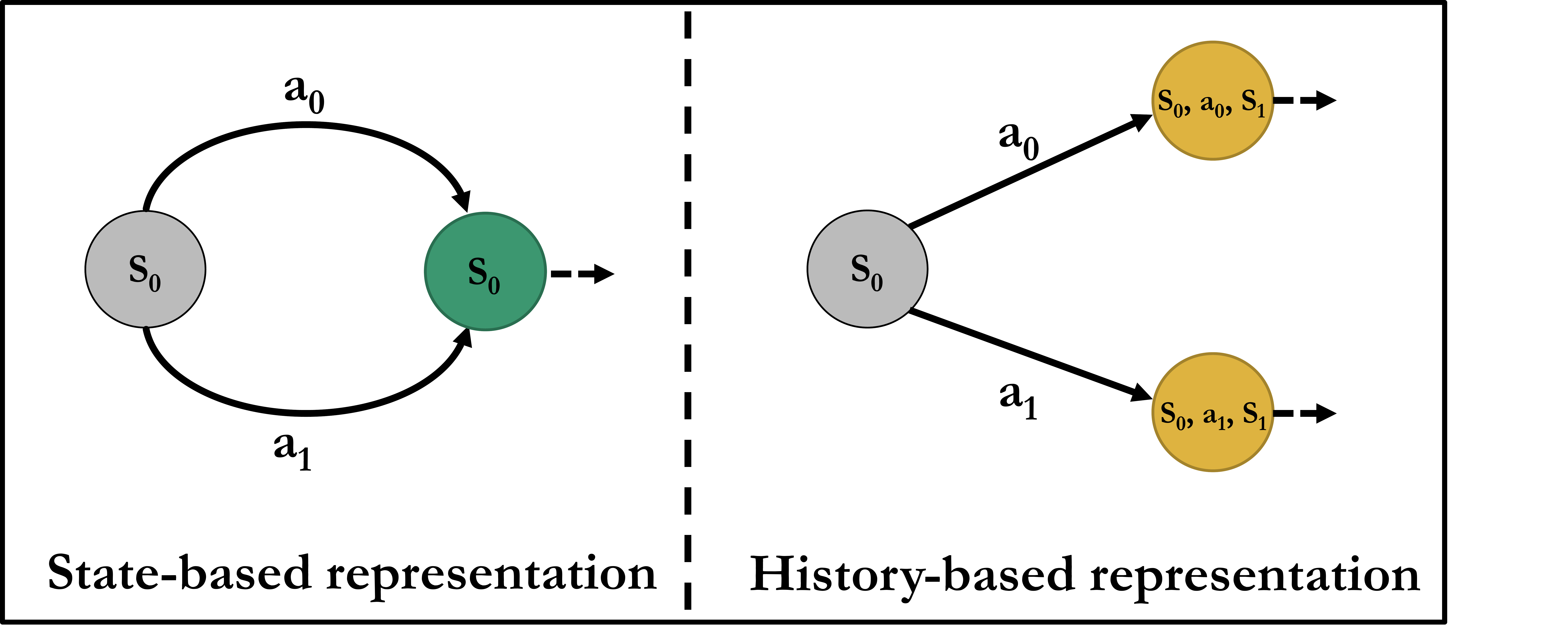} 
\caption{An ACNO-MDP example where a history-based method will perform poorly. Although actions $a_0$ and $a_1$ have the same effect on the state (left), in a history-based representation (right) they are not represented as equal.}
\label{fig:POMCPFailExample}
\end{figure}

\section{Proofs}
\label{sec:proofs}

In this appendix, we provide additional proofs for statements made in the paper, but for which we did not find space for a full proof.

\subsection{Proof of \cref{thm:MVOptimal}}
We start by proving \cref{thm:MVOptimal}, which intuitively states that for a policy in which control actions are chosen using the ATM-heuristic, choosing measurement according to $m_{\mv}$ (\cref{eq:MeasureValueCondition}) is optimal.
For convenience, we restate the lemma here:
\begin{lemma*}
    Given a fully known ACNO-MDP $\mathcal{M}$ with $\gamma < 1$.
    Define $\pi_{\text{ATM}}$ as in \cref{eq:piATM}, and $\pi'_{\text{ATM}}$ as:
    \begin{equation}
        \pi'_{\text{ATM}}(b) = \langle \max_{a \in A} Q(b,a), \psi(b) \rangle,
    \end{equation}
    with $\psi: b \rightarrow m$. For any choice of $\psi$, the following holds:
    \begin{equation}
    \label{eq:MVOptimal_apendix}
    V(\pi_{\text{ATM}}, \mathcal{M}) \geq V(\pi'_{\text{ATM}}, \mathcal{M})
    \end{equation}
\end{lemma*}

\begin{proof}
For our proof, we introduce $\pi_N$ as the policy which follows $\pi_{
\text{ATM}}$ for the first $N$ steps of an episode,
then follows $\pi'_{\text{ATM}}$. We note that  $\lim_{N \rightarrow \infty} \pi_N = \pi_{\text{ATM}}$, so the following is equivalent to our lemma:

\begin{equation}
    \lim_{N\rightarrow \infty} V(\pi_N, \mathcal{M}) \geq V(\pi'_{\text{ATM}}, \mathcal{M})
\end{equation}

We will now prove this equation holds via induction over N.
Trivially, we note the equation holds for $N=0$.

For $N=1$, after the current step, both policies are equal, meaning the difference in return between the two is determined only by the current action pair chosen.
We represent this as follows:
\begin{equation}
\label{eq:thm_QEquality}
    \forall b, \tilde{a}: Q^{\pi_1}(b, \tilde{a}) = Q^{\pi'_{\text{ATM}}}(b, \tilde{a}) \equiv Q(b,\tilde{a})
\end{equation}
For equal belief states $b$, we notice that both policies always choose the same control action, which we will denote as $a$.
If both also choose the same measuring action, trivially our lemma holds.
This means we only have to look at the cases where the two policies choose different measuring actions, i.e. where $m_{\text{MV}}(b) \neq \psi(b)$.

Suppose $m_{\text{MV}}(b,a) = 1 \neq \psi(b) = 0$, and suppose $m=0$ gives higher expected returns for this belief.
We represent the difference between the two expected returns using \cref{eq:bellmanMeasuring,eq:bellmanNotMeasuring}, where following  \cref{eq:thm_QEquality}, we may use the same Q-values for both.
This yields the following inequality:
\begin{equation}
\label{eq:bellmanDifference}
\begin{split}
    Q^{\pi_1}(b) & - Q^{\pi'_{\text{ATM}}}(b) = Q(b, \langle a, 1\rangle) - Q(b, \langle a, 0 \rangle ) \\
    & = -c + \gamma \sum_{s \in S} b'(s,a) \big[\bmax_{\tilde{a} \in \tilde{A}} Q(s, \tilde{a}) - Q(s,\tilde{a}_b)\big] \\
    & \leq 0
\end{split}
\end{equation}
We notice that the second line is exactly equal to the definition of measure value given in \cref{eq:measureRegret}.
Comparing the two, the inequality found here implies $\mv(b,a)$ must be negative.
However, following \cref{eq:MeasureValueCondition}, for negative $\mv(b,a)$,  $m_{\mv}(b,a)=0$, while we explicitely specified that $m_{\mv}(b,a)=1$.
We have found a contradiction, so we may assume that if $m_{\text{MV}}(b,a) = 1$, it is always optimal.
For $m_{\text{MV}}(b,a) = 0$, we can follow exactly the same logic, except we expect \cref{eq:bellmanDifference} to be positive instead, in which case we would again find a contradiction.
Thus, we conclude no situation exists where any $\psi$ gives a better return than $m_{\mv}$, proving our theorem for $N=1$.
Moreover, this proof holds starting from any initial belief state $b$, so we may state the following, more general, inequality:
\begin{equation}
\label{eq:lemma_N=1}
    \forall b: Q^{\pi_1}(b) \geq Q^{\pi'_{\text{ATM}}}(b)
\end{equation}

Now that we've proven $N=1$ as a baseline, we use induction for all $N>1$, where we may use the induction hypothesis that $\pi_{\text{ATM}}$ was optimal to follow upto this step, and need to prove that it is still optimal to follow for this step (assuming we follow $\pi'_{\text{ATM}}$ from step $N+1$ onwards).

We define $\mathcal{B}^{\pi}_N$ as the probability distribution over belief states after following policy $\pi$ for $N$ steps.
After $N-1$ steps, the policies $\pi_N$ and $\pi'_{\text{ATM}}$ have then ended up in belief states proportional to $\mathcal{B}^{\pi_N}_{N-1}$ and $\mathcal{B}^{\pi'_{\text{ATM}}}_{N-1}$.
Since via our induction hypothesis $\pi_{N}$ has been optimal for the first N steps, we may assume

If these distributions are not equal, then by our induction hypothesis the expected return for following $\pi'_{\text{ATM}}$ from $\mathcal{B}^{\pi_N}_{N-1}$ must be higher than that of $\mathcal{B}^{\pi'_{\text{ATM}}}_{N-1}$, or more formally:
\begin{equation}
    \E [ Q^{\pi_{\text{ATM}}}(b) | b \sim \mathcal{B}^{\pi_N}_{N-1} ] \geq \E [ Q^{\pi_{\text{ATM}}}(b) | b \sim \mathcal{B}^{\pi'_{\text{ATM}}}_{N-1} ].
\end{equation}
Using transitivity, this means if we prove our theorem for the case where both are equal, we have proven it for all cases.
Since in this case the probability of reaching all belief states are equal for both policies, it suffices to prove optimality if both policies end up in the same belief state $b_{N-1}$.
From $b_{N-1}$, there is only one step left for which $\pi_N$ follows $\pi_{\text{ATM}}$.
We notice, then, that our problem is now equivalent to proving that for any belief state $b_{N-1}$, following $\pi_1$ yields at least equal returns to following $\pi'_{\text{ATM}}$.
This we have already proven (\cref{eq:lemma_N=1})
\end{proof}

\subsection{Linearity in $a_b$}

In this section, we show that the second equality stated in \cref{eq:a_b} holds.
For this, we start by making our claim more explicit.
Writing out the equation for $a_b$ as give in \cref{eq:a_b} in the equation for $Q_{\text{ATM}}(b, \langle a, 0 \rangle)$ given in \cref{eq:bellmanNotMeasuring}, we find the following:
\begin{equation}
\label{eq:QLinearityNotMeasuring}
    Q_{\text{ATM}}(b, \langle a, 0 \rangle) = \hat{r}(b,a) + \gamma \bmax \sum_{s \in S} b'(s | b,a) Q_{\text{ATM}}(s, \tilde{a})
\end{equation}
Proving this equality holds is thus equivalent to proving the equality in \cref{eq:a_b}.
To prove that it holds, we start by giving the Q-value for not measuring in the standard way:
\begin{equation}
    Q_{\text{ATM}}(b, \langle a, 0 \rangle) = \hat{r}(b,a) + \gamma \bmax Q_{\text{ATM}}(b'(s,a), \tilde{a})
\end{equation}
Without loss of generality, let us assume we do not make any measurements for the next $N \in \mathbb{N}$ steps.
We can find the Q-value for any such sequence of $N$ steps via recursively applying $Q_{\text{ATM}}(b, \langle a, 0 \rangle)$, which (for $N\geq1$) looks as follows:
\begin{equation}
\begin{split}
    Q_{\text{ATM}} (b, \langle a, 0\rangle ) & = \hat{r}(b,a) +  \gamma \bmax_{\tilde{a} \in \tilde{A}} Q_{\text{ATM}}(b'(b,a), \tilde{a})\\
    & = \hat{r}(b,a)  +  \gamma [ \max_{a' \in A} \hat{r}(b'(b,a),a') \\
    &  + \gamma \bmax_{\tilde{a}' \in \tilde{A}} Q_{\text{ATM}}(b'(b'(b,a),a'),  \tilde{a}')] \\
    & = ... \\
\end{split}
\end{equation}
We can generalize this pattern as follows:
\begin{equation}
\label{eq:appendix_QNonmeasuringExact}
        Q_{\text{ATM}} (b, \langle a, 0\rangle ) = \sum_{t=0}^{t = N} \gamma^t \hat{r}(b_t,a_t) + \gamma^{N+1} Q(b_{N+1}, \tilde{a}_{N+1}),
\end{equation}
where we use $a_t$ and $\tilde{a}_t$ to denote the choosen (control) action after $t$ non-measuring steps after the current step, and $b_t$ the belief state for this step.

We rewrite our different representation of $Q_{\text{ATM}}(b, \langle a, 0 \rangle)$ (\cref{eq:QLinearityNotMeasuring}) in a similair fashion:
\begin{equation}
\begin{split}
 Q_{\text{ATM}} & (b, \langle a, 0 \rangle) = \hat{r}(b,a) + \gamma \bmax \sum_{s \in S} b'(s | b,a) Q_{\text{ATM}}(s, \tilde{a}) \\
    = & \hat{r}(b,a) + \gamma \max_{a' \in A} \sum_{s \in S} b'(s | b,a) [ \hat{r}(b'(s,a),a') \\
    & + \gamma \bmax \sum_{s' \in S} b'(s' | b'(s,a),a') Q_{\text{ATM}}(s', \tilde{a}') \\
    = & ...
\end{split}
\end{equation}
We notice that we can simplify this by using the fact that $\sum_{s \in S} \hat{r}(b(s),a) = \hat{r}(b, a)$.
Applied to our belief state, this gives the following equality:
\begin{equation}
    \sum_{s \in S} b'(s | b,a) \hat{r}(s,a') = \hat{r}(b'(b,a),a').
\end{equation}
Using this equality at every recursive step, we may generalize the pattern of our equation as follows:
\begin{equation}
\label{eq:appendix_QNonmeasuringApprox}
\begin{split}
    Q_{\text{ATM}}(b, \langle a, 0 \rangle) =  & \sum_{t = 0}^{t=N} \gamma^t r(b_t,a_t) \\
    & + \gamma^{N+1} \sum_{s \in S} b_{N+1}(s) Q_{\text{ATM}}(s, \tilde{a}_{N+1}).
\end{split}
\end{equation}
Comparing \cref{eq:appendix_QNonmeasuringExact,eq:appendix_QNonmeasuringApprox}, we notice that the summation over all rewards cancels out.
Any difference, then, can only be given by the last term.
If $N \neq \infty$, we take a measurement in this last term, in which case we may use $Q_{\text{ATM}}(b, \langle a, 1 \rangle)$ (\cref{eq:bellmanMeasuring}) for $Q_{\text{ATM}}(s, \tilde{a}_{N+1})$ in both equations.
Filling this in, we find they are indeed equal.
If instead we never take a measurement (i.e. $N = \infty$), then the contribution of this last term can be ignored, since $\lim_{N \rightarrow \infty}\gamma^{N} = 0$.
In this case, our equality holds as well.

\end{document}